%% file: main.tex
\documentclass[11pt]{article}
\usepackage[letterpaper,margin=1in]{geometry}
\usepackage{times,natbib}
\usepackage{amsmath,amssymb,amsthm,booktabs,longtable,graphicx}
\usepackage[T1]{fontenc}
\usepackage{hyperref,url}
\hypersetup{hidelinks}
\newtheorem{proposition}{Proposition}
\newtheorem{lemma}{Lemma}
\newcommand{\E}{\mathbb{E}}
\newcommand{\Prb}{\mathbb{P}}
\newcommand{\gap}{\operatorname{Gap}}
\title{Missingness-Aware Conformal Prediction\\Under Cross-Hospital Distribution Shift}
\author{Liang You$^{1,*}$, Dongwen Ou$^2$, Hengyu Shi$^3$, Siyuan Dai$^1$\\[0.5em]
\small $^1$University of Pittsburgh \quad $^2$Duke University\\
\small $^3$Xiamen University\\[0.3em]
\small $^*$Corresponding author: \texttt{liangyou03@pitt.edu}}
\date{September 2026}
\begin{document}
\maketitle

\begin{abstract}
Clinical measurements are recorded for some patients but not others, at rates that differ across hospitals, and marginal conformal coverage does not ensure coverage within groups defined by missingness. We propose a missingness-aware conformal calibration procedure for mortality prediction under cross-hospital distribution shift. It selects a measurement on an independent sample, groups patients by whether that measurement is recorded, and applies Mondrian calibration within each group, so no calibration outcome is reused. We evaluate the procedure across hospitals in eICU and across care units within one MIMIC-IV hospital, using three predictors. Relative to pooled calibration, it reduces the average worst-group coverage gap on its selected groups in all six settings, with a median reduction of 1.9 percentage points; paired site-bootstrap intervals exclude zero in five. These gains do not extend uniformly. Calibration by predicted risk achieves smaller gaps on a broader panel of missingness groups, and when eICU hospitals are evaluated separately, the gain shrinks for all three predictors and reverses in sign for one. We explain this discrepancy with a hospital-level decomposition. Pooling reweights hospitals through a covariance between group shares and coverage errors, and lets errors of opposite sign cancel: weighting explains the reversal, and cancellation accounts for most of the attenuation for the other two predictors. Constructed population distributions show that pooled and within-hospital evaluations can rank calibration methods oppositely even without sampling noise. Pooled improvement alone therefore cannot establish better coverage within hospitals, even when the calibration groups are fixed.
\end{abstract}

\section{Introduction}

Clinical records differ in which measurements they contain. A mortality model may receive a full laboratory panel for one patient and only vital signs for another. Which measurements are recorded carries predictive information \citep{Che2018} and reflects how care is delivered and documented \citep{Agniel2018}. Conformal prediction (CP) turns a predictor's outputs into prediction sets that contain the true outcome with a target probability, but its standard guarantee is \emph{marginal}: coverage is averaged over all patients. A single calibration threshold treats patients with and without a measurement alike, even when their prediction errors differ. Coverage can then meet its target on average while falling short for, say, patients without a recorded urine output.

Deployment across hospitals compounds the problem. Dataset shift is a recognized hazard for clinical prediction \citep{Finlayson2021}, and missingness can shift in two distinct ways. The proportion of patients missing a measurement can change between sites, and the prediction errors among patients with the same missingness status can change as well. Calibrating each missingness group separately prevents groups with different error distributions from sharing one threshold, but it cannot stabilize a group's error distribution at a new site. Marginal coverage, coverage within groups, and coverage within sites are therefore three different targets.

We propose \emph{missingness-aware conformal calibration}. Using a selection sample independent of the training and calibration data, the procedure chooses a measurement whose availability varies across sites, divides patients by whether it is recorded, and applies Mondrian calibration \citep{Vovk2005} within each group. Because selection uses no calibration outcomes, the usual Mondrian guarantee applies to these groups under exchangeability; calibrating a different partition, such as predicted risk, does not generally provide it. Group membership depends only on the mask, so it is also unchanged when the mortality predictor is updated. These properties make availability a practical basis for monitoring a specified clinical population.

We evaluate the procedure across hospitals in eICU and across care units within one MIMIC-IV hospital, using logistic regression (LR), gradient-boosted trees (XGBoost), and a multilayer perceptron (MLP). Relative to pooled calibration, missingness-aware calibration reduces the average worst-group coverage gap on its selected groups in all six dataset--predictor settings, with a median reduction of 1.9 percentage points. The gains do not extend uniformly, however. Calibration by predicted risk achieves smaller gaps on a broader panel of missingness groups, and when eICU hospitals are evaluated separately, the gain shrinks for all three predictors and reverses in sign for XGBoost.

To explain this discrepancy, we compare pooled and within-hospital evaluation directly. Pooling changes two things: it weights each hospital by its share of the group being evaluated, and it lets undercoverage at one hospital cancel overcoverage at another. We show that the weighting effect is a covariance between hospitals' group shares and their coverage errors, and decompose the difference between the two evaluations into weighting, cancellation, and changes in the worst group. Weighting accounts for the XGBoost reversal, whereas cancellation removes most of the LR and MLP gains. Constructed population distributions show that aggregation alone can reverse method rankings, without sampling noise. None of this analysis depends on how groups are defined, so it applies to any calibration partition evaluated across sites.

Our contributions are as follows.
\begin{itemize}
\item A missingness-aware calibration procedure that selects its partition on an independent sample and therefore avoids reusing calibration outcomes (Section~\ref{sec:theory}).
\item An analysis of pooled versus within-site coverage: an identity showing that pooling reweights sites by a covariance between group shares and coverage errors, an ordering of aggregation schemes, and a decomposition into weighting, cancellation, and worst-group effects (Section~\ref{sec:aggregation}).
\item A multi-site evaluation on eICU and MIMIC-IV showing that gains on the selected groups do not carry over uniformly to a broader panel of groups, to individual hospitals, or to outcome labels (Section~\ref{sec:experiments}).
\end{itemize}

\section{Related work}

\paragraph{Mondrian and missingness-aware conformal prediction.}
Conformal guarantees typically rely on exchangeability, that is, invariance of the joint distribution under reordering of observations. Mondrian methods calibrate within fixed categories under suitable exchangeability conditions \citep{Vovk2005}, and general treatments distinguish coverage validity from the informativeness of set size \citep{Lei2018,Angelopoulos2023}. \citet{Zaffran2023} obtain coverage conditional on the missingness mask through imputation and data augmentation. For multimodal regression, \citet{Azizi2026Missing} define Mondrian groups by modality availability or by disagreement among predictions. We build on this use of availability, but study binary clinical prediction across sites and distinguish the groups used for calibration from the groups used for evaluation.

\paragraph{Conditional coverage.}
\citet{Gibbs2025} guarantee coverage under every covariate shift in a user-specified function class; we use their official implementation with a finite basis. \citet{Martinez2024} learn interpretable groups with similar score distributions, and \citet{Azizi2026CLEAR} combine epistemic uncertainty about the model with aleatoric variability in the outcome. Our score-tree comparator is a simple grouping control, not the full procedure of \citet{Martinez2024}. Conditional validity \citep{Vovk2012}, distribution-free impossibility results \citep{Barber2021}, and guarantees conditional on the training data \citep{Bian2023} concern different probability statements; our coverage bound averages over calibration scores. Other methods adapt the score itself, for example conformalized quantile regression \citep{Romano2019} and adaptive or regularized classification sets \citep{Romano2020,Angelopoulos2021}; set-valued classifiers balance error control and size \citep{Sadinle2019}. We hold the score fixed to isolate the effect of the calibration groups. \citet{Dabah2025} study how temperature scaling affects conformal classifiers, including the $1-\widehat p_y$ score used here; we instead examine binary scores within fixed clinical calibration cells, where monotone-score invariance separates odds shifts from symmetric temperature changes.

\paragraph{Evaluating coverage.}
\citet{Braun2026Diagnostics} estimate conditional coverage by classification and distinguish undercoverage from overcoverage, and \citet{Zhou2026Assessment} develop learned criteria for assessment and selection. We instead hold clinical groups fixed and ask how conclusions change with the audited groups, the hospital weights, and the order of aggregation. Hierarchical conformal methods address dependence among grouped observations \citep{Lee2026Hierarchical}; our hospital-level comparisons are descriptive and do not provide hierarchical coverage guarantees.

\paragraph{Distribution shift.}
Under covariate shift, weighted conformal prediction reweights calibration scores by input density ratios \citep{Tibshirani2019}, and coverage loss beyond exchangeability can be bounded \citep{Barber2023}. Localized conformal prediction emphasizes nearby observations \citep{Guan2023}, and adaptive conformal inference updates calibration online \citep{Gibbs2021}. For missing data, \citet{Fan2025} combine imputation, masking, and weighting to obtain mask-conditional coverage. Our weighting comparator uses only the frequencies of two mask states and is not an implementation of their method.

\section{Missingness-aware conformal calibration}
\label{sec:theory}

\subsection{Setup}

Let $X$ be a raw record with $d$ variables, some of which may be missing, and let $M\in\{0,1\}^d$ be its missingness mask, with $M_j=1$ if variable $j$ is missing. Preprocessing maps the record to $Z=\Phi(X,M)\in\mathbb{R}^p$. The outcome $Y\in\{0,1\}$ indicates death ($Y=1$), and a fixed predictor $f(Z)$ estimates its probability. We use the nonconformity score
\begin{equation}
 S(X,M,y)=y\{1-f(Z)\}+(1-y)f(Z),
\end{equation}
which is small when the prediction agrees with the candidate label $y$. For a threshold $q$, the prediction set is $\Gamma=\{y\in\{0,1\}:S(X,M,y)\le q\}$. It may contain zero, one, or two labels; we do not force empty sets to include a label. The set covers the true outcome exactly when the true-label score satisfies $S(X,M,Y)\le q$. All theoretical statements condition on the fitted predictor, the preprocessing, and the grouping rule, each of which must be independent of the calibration sample.

With target coverage $1-\alpha$ and the true-label scores of $n$ calibration patients, pooled split conformal prediction sets $q$ to the $r$th smallest score, where $r=\lceil(n+1)(1-\alpha)\rceil$, and to $+\infty$ if $r>n$. Given a fixed group map $G(M)\in\{1,\ldots,K\}$, Mondrian calibration applies the same rule separately to the $n_k$ scores in each group $k$. An empty calibration group receives $q=+\infty$ and hence both labels. We use $\alpha=0.1$ unless stated otherwise.

\subsection{Selecting the partition without reusing calibration data}

Missingness-aware calibration fixes $G$ before calibration, using a selection sample drawn from sites disjoint from those used for training and calibration. A candidate measurement must be nonconstant in the training data and, in the selection sample, missing in $10\%$ to $90\%$ of records with at least 100 records in each state. Hospital metadata are excluded. Let $\mathcal J$ denote the eligible measurements and $\widehat\pi_{hj}$ the missing proportion of measurement $j$ at selection site $h$. We choose the measurement whose missingness varies most across sites,
\begin{equation}
\label{eq:select}
 j^*=\mathop{\arg\max}_{j\in\mathcal J}
 \Bigl\{\max_h\widehat\pi_{hj}-\min_h\widehat\pi_{hj}\Bigr\},\qquad
 G(M)=M_{j^*},
\end{equation}
with ties broken deterministically. Because $G$ is fixed before any calibration outcome is observed, the standard Mondrian analysis applies conditionally on $G$, with no correction for data-dependent group selection: if calibration and test records are exchangeable within each group, then $\Prb(Y\in\Gamma\mid G=g)\ge1-\alpha$ for every group $g$ with calibration data \citep{Vovk2005}. Under shift between sites this guarantee no longer holds, so we evaluate coverage empirically.

The selection criterion is a heuristic. It targets the measurement whose availability depends most on the site and yields an auditable partition, not an optimal one. We therefore compare it with a random eligible measurement and with a prespecified lactate indicator. Since $G$ depends only on the mask, replacing the predictor leaves group membership unchanged, although the new scores require fresh calibration. A risk-based partition retains its membership only if its grouping predictor is frozen.

\paragraph{What separate calibration can and cannot correct.}
Appendix~\ref{app:proof} bounds the expected coverage error by four terms: differences between the score distributions of the groups, shift of a group's score distribution between calibration and test sites, finite calibration samples, and ties. Mondrian calibration removes the first term but not the second. Missingness frequencies identify neither term, so the bound motivates separate calibration but does not rank candidate partitions by their actual errors.

\section{Pooled versus within-site coverage}
\label{sec:aggregation}

Coverage gains are usually reported after pooling test patients across sites. This section shows why such gains need not hold within sites by separating the two ways in which pooled and within-site evaluation differ: the weights given to sites and the order of aggregation. The results hold for any fixed calibration partition.

Consider test sites $h=1,\ldots,H$ and a partition into groups $g=1,\ldots,K$. Let $n_{hg}>0$ be the number of test patients in cell $(h,g)$, with $n_h=\sum_g n_{hg}$, $n_g=\sum_h n_{hg}$, and $n=\sum_h n_h$. We restrict this comparison to sites where every group is observed. For a given calibration method, let $\widehat C_{hg}$ be the fraction of test patients in cell $(h,g)$ whose set covers the true outcome, and $e_{hg}=\widehat C_{hg}-(1-\alpha)$ its signed error. For site weights $w_h\ge0$ with $\sum_h w_h=1$, define
\begin{equation}
\label{eq:adb}
 A_w=\max_g\Bigl|\sum_h w_he_{hg}\Bigr|,\qquad
 D_w=\max_g\sum_h w_h|e_{hg}|,\qquad
 B_w=\sum_h w_h\max_g|e_{hg}|.
\end{equation}
All three use the same site weights for every group, which we call \emph{common weights}, and differ only in the order of aggregation. $A_w$ averages signed errors across sites and then takes the worst group; $D_w$ averages absolute errors and then takes the worst group; $B_w$ takes the worst group within each site and then averages, and is therefore the within-site gap. Pooling patients across sites instead gives the pooled gap
\begin{equation}
 \gap_{\rm pool}=\max_g\Bigl|\sum_h\omega_{hg}e_{hg}\Bigr|,\qquad \omega_{hg}=n_{hg}/n_g,
\end{equation}
which uses a different set of site weights for each group.

\begin{proposition}[Pooling and aggregation]
\label{prop:aggregation}
Let $w_h=n_h/n$ be patient weights, $p_{hg}=n_{hg}/n_h$ the share of group $g$ at site $h$, and $\bar p_g=\sum_h w_hp_{hg}$.
\begin{enumerate}
\item[(i)] $\omega_{hg}=w_hp_{hg}/\bar p_g$. Hence the pooled weights coincide with the common weights, $\omega_{hg}=w_h$ for every $g$ and every $h$ with $w_h>0$, if and only if each group's proportion is the same at every such site.
\item[(ii)] For every group $g$, the pooled signed error differs from its common-weight counterpart by a covariance across sites:
\begin{equation}
\label{eq:cov}
 \sum_h\omega_{hg}e_{hg}=\sum_hw_he_{hg}+\operatorname{Cov}_w\!\left(\frac{p_{hg}}{\bar p_g},\,e_{hg}\right),
\end{equation}
where $\operatorname{Cov}_w(a,b)=\sum_hw_h(a_h-\bar a)(b_h-\bar b)$ with $w$-weighted means $\bar a,\bar b$.
\item[(iii)] For any common weights $w$, $A_w\le D_w\le B_w$.
\end{enumerate}
\end{proposition}
\begin{proof}
(i) Dividing the numerator and denominator of $\omega_{hg}=n_hp_{hg}/\sum_{h'}n_{h'}p_{h'g}$ by $n$ gives the identity. Also, $\omega_{hg}=w_h$ exactly when $p_{hg}=\bar p_g$.

(ii) Let $a_h=p_{hg}/\bar p_g$, so $\sum_hw_ha_h=1$. By (i),
\[
\sum_h\omega_{hg}e_{hg}-\sum_hw_he_{hg}
=\sum_hw_h(a_h-1)e_{hg}
=\operatorname{Cov}_w(a,e_{\cdot g}),
\]
where the last equality uses $\sum_hw_h(a_h-1)=0$.

(iii) The first inequality is the triangle inequality. For every $g$,
$\sum_hw_h|e_{hg}|\le\sum_hw_h\max_{g'}|e_{hg'}|$, which gives the second.
\end{proof}

The same statements hold with population coverages and probabilities in place of their empirical versions. For two calibration methods, write $\Delta$ for the difference between them. Moving from the pooled comparison to the within-site comparison then decomposes as
\begin{equation}
\label{eq:decomp}
 \Delta B_w-\Delta\gap_{\rm pool}
 =\underbrace{(\Delta A_w-\Delta\gap_{\rm pool})}_{\text{weighting}}
 +\underbrace{\Delta(D_w-A_w)}_{\text{cancellation}}
 +\underbrace{\Delta(B_w-D_w)}_{\text{worst-group switching}}.
\end{equation}
The weighting term replaces pooled weights by common weights. With patient weights, it arises from the covariance in \eqref{eq:cov}: pooling gives more weight to the sites where a group is common, so each group's pooled error is pulled toward its errors at those sites. The term vanishes if group proportions are constant across sites, or if each group's error is constant across sites under both methods; with equal site weights, it also includes the change from patient to equal weights. The cancellation term measures how much of a method's small $A_w$ comes from errors of opposite sign offsetting one another across sites, and the last term reflects that different sites can have different worst groups. Each method satisfies $A_w\le D_w\le B_w$, but differences between methods can have either sign, so pooled and within-site evaluation can rank methods differently.

Groupwise calibration targets coverage in each group over the calibration population; it does not constrain coverage at each deployment site, so a method's pooled gap can decrease while its within-site gap increases. Varying group proportions are necessary for the weighting term under patient weights but not sufficient for a reversal, which also depends on the coverage errors and their association with those proportions. Our selector deliberately chooses a measurement whose availability varies across sites, so site weights deserve particular scrutiny when it is evaluated; it does not, however, maximize the covariance in \eqref{eq:cov}, which depends on errors at the test sites.

\paragraph{Reversals without sampling noise.}
Because Proposition~\ref{prop:aggregation} holds for population quantities, reversals are not merely artifacts of noisy empirical maxima. Appendix~\ref{app:hospital-mechanism} constructs known score distributions for two hospitals and two groups. Even with no shift between calibration and test populations, exact coverage in every group coexists with nonzero errors within hospitals, and changing hospital weights yields opposite method rankings before and after nonlinear aggregation. With 25 test observations per cell, finite-sample comparisons frequently disagree with the population signs. These examples establish possible mechanisms; they do not identify the cause of any particular clinical reversal.

\section{Clinical evaluation}
\label{sec:experiments}

\subsection{Cohorts and study design}

We retrospectively predict in-hospital mortality from information recorded during the first day in the intensive care unit (ICU). The \emph{eICU cohort} uses GOSSIS-1-eICU v1.0.0, a release derived from the eICU Collaborative Research Database \citep{Raffa2022,Pollard2018,GOSSIS2022}. We use its minimally cleaned table, which retains missing values, rather than the imputed version. Restricting to hospitals with at least 500 records yields 106,488 records from 82 hospitals, with one record per released patient identifier. Supplied mortality predictions, discharge outcomes, length of stay, and identifiers are excluded; screening on training data retains 183 raw predictors.

The \emph{MIMIC-IV cohort} uses ICU stays from MIMIC-IV v3.1 \citep{Johnson2024,Johnson2023}. Of 94,458 adult stays with first-day data, 85,181 belong to the six care units with at least 10,000 stays. Selecting one eligible stay per patient at random yields 58,879 patients in all six units, described by 153 nonconstant raw predictors. Shifts in this cohort are between care units \emph{within one hospital}.

Each assignment divides sites into four disjoint roles: training the predictor, selecting the grouping, calibrating thresholds, and testing. eICU uses ten assignments with approximately $60/10/10/20\%$ of hospitals in these roles. MIMIC-IV holds out each of the six units for testing and allocates the remaining five by three fixed rotations (two for training, two for selection, one for calibration), giving 18 assignments. No site or patient appears in two roles within an assignment. Different assignments share sites, so they are not independent replications.

For each dataset, predictor, and assignment, all calibration methods share the same predictions, calibration patients, and test patients. The predictors are LR, XGBoost, and an MLP with one hidden layer. Feature selection, category encoding, imputation, and scaling are fitted on training data only, and the grouping rule uses selection data only. Appendix~\ref{app:implementation} gives model settings and convergence diagnostics. Because predictors use first-day data, this task is neither prediction at admission nor a test of clinical safety.

\subsection{Comparators and evaluation targets}

We compare seven calibration rules that share the same score and differ only in how calibration scores are grouped or weighted:
\emph{pooled} calibration;
\emph{missingness-aware} calibration (ours);
two partition controls, a \emph{random} eligible measurement and \emph{lactate} availability;
\emph{risk terciles}, with cut points at the selection-sample terciles of predicted probability;
a \emph{score tree} with at most three leaves that predicts the nonconformity score from predicted risk and overall missingness, the fraction of retained raw predictors that are absent;
and \emph{mask weighting}, which reweights calibration scores by the estimated ratio of target to calibration frequencies of the two selected mask states (Appendix~\ref{app:implementation}).
Mask weighting uses unlabeled target inputs; the other six rules use no target data. We additionally include binary risk groups and \emph{label-only} calibration in Table~\ref{tab:core}. Label-only calibration fits one threshold per outcome and compares each candidate label with its own threshold; it never uses the true test label to construct a set.

All methods are evaluated on the same targets within an assignment. The \emph{selected gap} is the largest absolute deviation from target coverage across the two selected groups,
\begin{equation}
 \gap_{\rm sel}=\max_{g\in\{0,1\}}
 \left|\widehat{\Prb}(Y\in\Gamma\mid M_{j^*}=g)-0.9\right|.
\end{equation}
The \emph{panel gap} takes the same maximum over both states of \emph{every} measurement in $\mathcal J$, and therefore audits groups beyond those used for calibration. Panels contain 83--106 measurements in eICU and 46--56 in MIMIC-IV, depending on the assignment; the smallest audit groups in the full test samples contain 775 and 189 patients. Panel groups overlap and do not cover combinations of missing variables or demographic attributes. Because absolute gaps penalize both undercoverage and overcoverage, we also report marginal coverage and average set size.

We also compare with the official implementation of conditional conformal prediction \citep{Gibbs2025}, using a basis of three functions: a constant, the predicted probability $f(Z)$, and overall missingness. This comparison uses XGBoost, the full calibration sample, and 2,000 randomly sampled test patients per assignment; every other method is reevaluated on the same patients, and these results are reported separately from full-test results.

\subsection{Gains on the selected groups}

Missingness-aware calibration reduces the observed mean selected gap relative to pooling in all six settings (Table~\ref{tab:core}). Reductions range from 0.34 to 3.13 percentage points (pp), with a median of 1.87 pp. Five of the six paired site-bootstrap intervals in Table~\ref{tab:core} exclude zero, also after adjustment for six comparisons (Appendix~\ref{app:paired}). The exception is eICU XGBoost, where the two rules produce different sets for only 1.03\% of test patients on average (Appendix~\ref{app:set-difference}); improvement in all six means is therefore not evidence of improvement in all six populations. MIMIC-IV intervals rest on only six units and serve as sensitivity summaries. Figure~\ref{fig:paired} shows variation across assignments.

\begin{table}[t]
\caption{Mean selected gap on full test samples. Risk (2/3): two/three predicted-risk groups; Label: outcome-only calibration. Gain is pooled minus missingness in percentage points. Paired 95\% basic bootstrap intervals condition on fitted models and assignments; MIMIC-IV uses only six units. Ten eICU and 18 MIMIC-IV assignments share predictions and evaluation groups.}
\label{tab:core}
\centering\footnotesize\setlength{\tabcolsep}{3pt}
\input{tables/core_table.tex}
\end{table}

\begin{figure}[t]
\centering
\includegraphics[width=\linewidth]{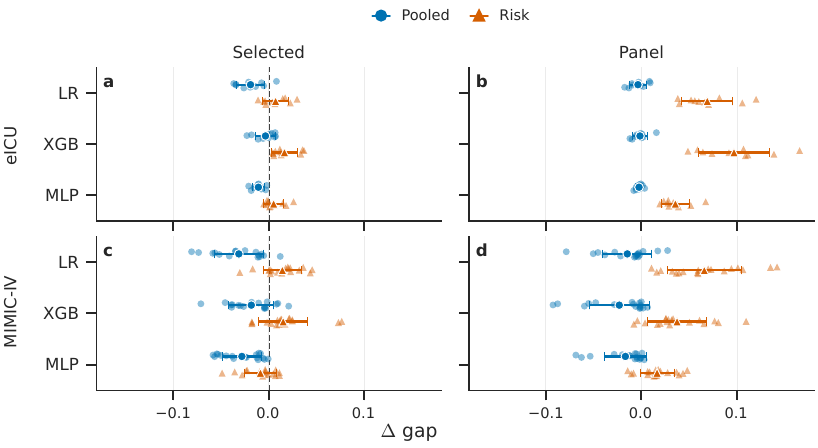}
\caption{Coverage differences on the selected groups and the audit panel. $\Delta$ gap is the gap of missingness-aware calibration minus that of the indicated comparator; negative values favor missingness. Pooled: pooled calibration; Risk: risk terciles. Small points are individual assignments (ten eICU, 18 MIMIC-IV); large markers and bars show means $\pm$ one SD. Assignments overlap, so bars are not confidence intervals.}
\label{fig:paired}
\end{figure}

\subsection{Do the gains extend to other missingness groups?}

\paragraph{Risk partitions.}
Risk terciles achieve smaller panel gaps than missingness-aware calibration for every dataset and predictor. For eICU XGBoost, the full-test panel gap is 0.020 with risk terciles versus 0.116 with missingness, with average set sizes of 1.026 and 0.952; for MIMIC-IV XGBoost, the gaps are 0.042 and 0.079, with sizes 1.111 and 1.026. Better panel coverage thus comes with larger sets. An average size below one reflects empty sets, not greater clinical efficiency.

Risk terciles also achieve a smaller selected gap in five of six settings. The advantage does not come from using more thresholds: binary risk groups, which use two thresholds as missingness calibration does, still have smaller panel gaps in all six settings and smaller selected gaps in three (0.029 and 0.012 for eICU XGBoost). It also persists when the panel is restricted to groups of at least 500 or 1,000 patients (Appendix~\ref{app:panel-size}). Label-only calibration, in contrast, does not reduce the selected gaps relative to pooling (Table~\ref{tab:core}). Different partitions nevertheless target different populations: risk calibration does not generally guarantee coverage in missingness groups, and its membership changes whenever the grouping predictor is updated.

\paragraph{Partition controls.}
The controls also limit the case for selecting by variation in missingness. On eICU XGBoost, the selected gap is 0.026 for our rule, 0.022 for a random eligible measurement, and 0.020 for lactate availability. On MIMIC-IV, our rule has a smaller mean selected gap than both controls for all three predictors. Because these selected gaps are evaluated on the groups chosen by our rule, they do not rank candidate partitions without bias. On the panel, random and clinical partitions can be competitive, while risk terciles remain stronger. The evidence supports a transparent heuristic, not an optimal selector.

\paragraph{Conditional conformal prediction.}
On the matched eICU subsamples, conditional conformal prediction improves the panel gap over missingness-aware calibration (0.055 versus 0.125) but slightly worsens the selected gap (0.036 versus 0.031). On MIMIC-IV, it improves both gaps, with larger sets. Adding the selected mask indicator to its basis, without further basis search, reduces the selected gap from 0.036 to 0.025 in eICU but increases it from 0.038 to 0.040 in MIMIC-IV; aligning the basis with the evaluation target does not yield uniform gains. Table~\ref{tab:aligned} reports both panel gap and shortfall for all four methods on these same subsamples. Subsample gaps are not comparable with full-test gaps, since their smallest groups contain only 65 eICU and 40 MIMIC-IV patients.

\subsection{Do the gains extend to individual hospitals?}

The gains above pool patients across test hospitals. To evaluate them within hospitals, we compute each eICU hospital's selected gap, average its repeated evaluations, and then average across the 76 hospitals tested at least once. Missingness-minus-pooled differences are $-0.0045$, $+0.0023$, and $-0.0004$ for LR, XGBoost, and MLP. All three are smaller in magnitude than the corresponding pooled gains, and only XGBoost has the opposite sign. Appendix~\ref{app:implementation} reports sensitivity of the MLP comparison to the training cap.

Paired resampling of hospitals gives 95\% intervals of $[-0.0083,-0.0006]$, $[-0.0002,0.0047]$, and $[-0.0030,0.0023]$, respectively. All three eICU intervals include zero after adjustment for six comparisons. The XGBoost reversal is therefore a point-estimate disagreement, not established population harm; its pooled gain is itself not distinguishable from zero (Table~\ref{tab:core}). These intervals condition on fitted models and assignments. MIMIC-IV tests one unit per assignment, so pooled and within-site evaluation coincide: differences are $-0.0313$, $-0.0183$, and $-0.0287$. Its six units provide a small, within-institution sensitivity analysis, not an independent test of hospital shift.

\paragraph{Weighting versus cancellation.}
Table~\ref{tab:aggregation} applies the decomposition of Section~\ref{sec:aggregation} to a fixed set of hospitals in each assignment, those containing both selected groups (at most three are excluded). Its first column is ordinary pooling; the others use common weights and aggregate before ($A_w$) or after ($B_w$) taking each hospital's worst group. For XGBoost, the difference turns from $-0.0052$ to $+0.0019$ as soon as pooled weights are replaced by common patient weights, before the aggregation order changes: the reversal is a weighting effect, and the covariance in \eqref{eq:cov} reproduces it exactly (Table~\ref{tab:covariance}). For LR and MLP, the sign survives reweighting, and the change comes from the aggregation order. Moving from $A_w$ to $B_w$ shrinks the LR difference from $-0.0178$ to $-0.0072$ and the MLP difference from $-0.0077$ to $-0.0007$, almost entirely through cancellation ($+0.01135$ and $+0.00631$); worst-group switching is small. Equal weights give the same conclusions (Appendix~\ref{app:clinical-decomposition}).

\begin{table}[t]
\caption{Missingness minus pooled gap in eICU on matched hospitals, averaged across assignments; negative values favor missingness. Pre: signed errors aggregated before the absolute maximum ($A_w$). Post: each hospital's gap computed first ($B_w$). Group weights: ordinary patient pooling within each group. All columns use the same hospitals within an assignment.}
\label{tab:aggregation}
\centering\footnotesize\setlength{\tabcolsep}{3pt}
\input{tables/aggregation_table.tex}
\end{table}

\subsection{Do the gains extend to outcome labels?}

Coverage can also differ sharply between outcomes. For eICU XGBoost, missingness calibration covers only 16.1\% of patients who died, meaning that their sets contain death, despite overall coverage of 0.895; with mortality below 11\%, errors can concentrate on deaths while overall coverage stays near target. This low value depends on the probability scale rather than on risk ranking: multiplying predicted odds by $1/4$ or 4 leaves rankings and AUROC unchanged but yields death coverage of 0.024 or 0.593 after recalibration, while overall coverage stays between 0.894 and 0.896. Symmetric temperature changes, in contrast, leave every set unchanged (Appendix~\ref{sec:invariance}). Calibrating risk and outcome jointly raises death coverage to 0.908 in eICU and 0.906 in MIMIC-IV, but returns both labels for 63.7\% and 61.1\% of patients, respectively (Appendix~\ref{app:outcome-detail}).

\section{Discussion and conclusion}

Missingness-aware calibration reduces the average gap on its selected groups in all six settings. Risk grouping nevertheless has smaller gaps on the broader panel, including after restricting group sizes. Within eICU hospitals, the gain shrinks for every predictor: the XGBoost point comparison reverses through weighting, whereas the LR and MLP gains are largely removed by cancellation.

These findings do not argue against missingness-aware calibration. When missingness groups are the population that must be covered, it guarantees their coverage under exchangeability, which risk grouping does not, and it improves their pooled coverage in every setting we studied. What the findings argue against is inferring site-level reliability from a pooled improvement. Because the decomposition does not depend on how groups are defined, and the population construction shows that aggregation alone can reverse rankings, the same caution applies to any calibration partition whose group proportions vary across sites. Multi-site evaluations should therefore report the audited groups, hospital weights, and aggregation order, alongside outcome coverage and prediction-set composition.

\paragraph{Limitations.}
The study is retrospective and uses information unavailable at admission. Its coverage bounds assume independent calibration observations, not within-hospital dependence, and the site-bootstrap intervals condition on fixed models and assignments. The clinical reversal is not statistically significant, and MIMIC-IV covers units of one institution. The selector is a heuristic; the conditional comparator covers one predictor and two bases. We do not benchmark CP-MDA \citep{Zaffran2023}, which masks inputs and recomputes predictions rather than recalibrating fixed scores; our comparisons are therefore limited to methods that keep the prediction pipeline fixed. Clinical usefulness also requires a policy for empty and ambiguous sets.

\section*{AI-use statement}
AI tools assisted with writing and grammar checks. The authors are responsible for verifying the results, citations, and final manuscript.

\section*{Ethics statement}
This retrospective study uses clinical databases that require credentialed access and impose data-use restrictions; access was obtained through the required credentialing process. We do not release individual patient data, identifiers, or predictions. Missingness can reflect care processes and unequal access to measurements. Calibrating a missingness group neither establishes demographic fairness nor justifies withholding measurements. Clinical deployment and patient benefit are not evaluated, and empty prediction sets would require a separate operational policy. Data-access authorization is distinct from ethics review; no study-specific institutional approval or exemption is claimed.

\section*{Reproducibility statement}
Section~\ref{sec:theory} states the assumptions and calibration rules, Section~\ref{sec:aggregation} states and proves the aggregation results, and Appendix~\ref{app:proof} proves the coverage bound. The appendices specify preprocessing, model settings, selection, weighting, uncertainty summaries, and all controlled distributions. Individual clinical records and predictions are withheld under the original access restrictions, so independent reproduction requires authorized access to the databases.

\bibliography{references}
\bibliographystyle{plainnat}

\clearpage
\input{appendix/supplement}
\end{document}

%% file: tables/core_table.tex
\begin{tabular}{llrrrrrl}
\toprule
Data & Model & Pooled & Missing & Risk (2) & Risk (3) & Label & Gain [95\% interval], pp \\
\midrule
eICU & LR & 0.040 & 0.021 & 0.018 & 0.014 & 0.040 & 1.91 [1.73, 3.17] \\
eICU & XGB & 0.029 & 0.026 & 0.012 & 0.009 & 0.030 & 0.34 [-0.08, 1.07] \\
eICU & MLP & 0.029 & 0.017 & 0.017 & 0.014 & 0.029 & 1.17 [1.09, 1.94] \\
MIMIC-IV & LR & 0.083 & 0.052 & 0.053 & 0.038 & 0.086 & 3.13 [1.76, 4.52] \\
MIMIC-IV & XGB & 0.060 & 0.041 & 0.027 & 0.026 & 0.090 & 1.83 [1.33, 2.46] \\
MIMIC-IV & MLP & 0.058 & 0.029 & 0.039 & 0.038 & 0.059 & 2.87 [1.83, 3.97] \\
\bottomrule
\end{tabular}

%% file: tables/aggregation_table.tex
\begin{tabular}{lrrrrr}
\toprule
Model & Group weights & Patient/pre & Patient/post & Equal/pre & Equal/post \\
\midrule
LR & -0.0208 & -0.0178 & -0.0072 & -0.0159 & -0.0046 \\
XGB & -0.0052 & +0.0019 & +0.0020 & +0.0028 & +0.0023 \\
MLP & -0.0124 & -0.0077 & -0.0007 & -0.0066 & -0.0004 \\
\bottomrule
\end{tabular}

%% file: appendix/supplement.tex
\appendix

\section{Terminology and notation}
\label{app:notation}

This section collects the terms used in the main text and the appendices.

\paragraph{Prediction sets and coverage.}
For a record with features $(X,M)$ and fitted death probability $f(Z)$, the nonconformity score of a candidate label $y\in\{0,1\}$ is $S(X,M,y)=y\{1-f(Z)\}+(1-y)f(Z)$, one minus the predicted probability of $y$. Given a threshold $q$ estimated from calibration data, the prediction set is $\Gamma=\{y:S(X,M,y)\le q\}$. The set contains the true outcome if and only if the \emph{true-label score} $S=S(X,M,Y)$ satisfies $S\le q$, so every coverage statement below is a statement about the distribution of $S$ relative to $q$. A set can be empty ($\varnothing$), a singleton ($\{0\}$ predicts survival and $\{1\}$ predicts death), or ambiguous ($\{0,1\}$).

\paragraph{Sites, roles, and assignments.}
A site is a hospital in eICU or a care unit in MIMIC-IV. An \emph{assignment} partitions sites into four disjoint roles: fitting the predictor, selecting the grouping, estimating calibration thresholds, and testing. There are ten eICU and 18 MIMIC-IV assignments; with three predictors, this gives $28\times3=84$ fitted predictors. Assignments share sites, so they are not independent replications. Unless stated otherwise, each quantity is computed within an assignment and then averaged across assignments.

\paragraph{Predictors.}
LR is logistic regression. XGBoost is an ensemble of gradient-boosted decision trees \citep{Chen2016}. The MLP (multilayer perceptron) is a feedforward neural network with one hidden layer. Because a tree ensemble returns finitely many distinct probabilities, its scores can be tied. The area under the receiver operating characteristic curve (AUROC) is the probability that a randomly chosen death receives a higher predicted risk than a randomly chosen survivor; it depends only on the ranking of predicted risks.

\paragraph{Evaluation quantities.}
For a group $g$ of test patients, $\widehat C_g$ is the fraction whose prediction set contains the true outcome. The \emph{selected gap} is $\max_g|\widehat C_g-0.9|$ over the two states of the selected measurement, and the \emph{panel gap} is the same maximum over both states of every eligible measurement. Over whichever groups a table specifies, $U=\max_g(0.9-\widehat C_g)_+$ is the largest shortfall and $O=\max_g(\widehat C_g-0.9)_+$ the largest excess, with $(x)_+=\max(x,0)$. We write $C$ for overall coverage, $C_0$ and $C_1$ for coverage among survivors and among deaths, and $\mu=\Prb(Y=1)$ for mortality.

\section{Coverage bounds}
\label{app:proof}

\subsection{Setting and result}

We condition throughout on the fitted predictor, the preprocessing map, and the group map $G$, all of which are independent of the calibration sample. Let $F$ be the cumulative distribution function (CDF) of the true-label score $S$ in the calibration population, $F_k$ its CDF within group $k$, and $H_k$ its CDF within group $k$ in the deployment population. For groups with positive probability in both populations, define the Kolmogorov distances
\begin{equation}
 \eta_k=\sup_t|F_k(t)-F(t)|,\qquad
 \delta_k=\sup_t|H_k(t)-F_k(t)|.
\end{equation}
The first measures how far group $k$'s calibration scores differ from the pooled calibration scores; the second measures how far they shift between calibration and deployment. Neither is determined by missingness frequencies. To allow ties, let $a=\sup_t\Prb_{\rm cal}(S=t)$ and $a_k=\sup_t\Prb_{\rm cal}(S=t\mid G=k)$ be the largest probability masses at a single score value (atoms) for pooled scores and for group $k$.

\begin{proposition}[Expected coverage with possible ties]
\label{prop:coverage}
Suppose the $n$ calibration examples are independent and identically distributed (iid), and the deployment example is independent of them. Pooled calibration satisfies
\begin{equation}
 \left|\Prb\{Y_{\rm te}\in\Gamma_{\rm pool}\mid G_{\rm te}=k\}-(1-\alpha)\right|
 \le \eta_k+\delta_k+\frac{1}{n+1}+a.
 \label{eq:pool}
\end{equation}
Let $G_{1:n}$ denote the group labels of all calibration examples. Conditional on these labels, for $n_k\ge1$, groupwise calibration satisfies
\begin{equation}
 \left|\Prb\{Y_{\rm te}\in\Gamma_{\rm group}\mid G_{\rm te}=k,G_{1:n}\}-(1-\alpha)\right|
 \le \delta_k+\frac{1}{n_k+1}+a_k.
 \label{eq:group}
\end{equation}
Here $\Gamma_{\rm pool}$ and $\Gamma_{\rm group}$ are the sets produced by the two rules, and the probabilities average over calibration scores and the independent test example. For continuous score distributions, the atom terms vanish.
\end{proposition}

Separate calibration removes the pooling term $\eta_k$ at the cost of a larger sample-size term, $1/(n_k+1)$ rather than $1/(n+1)$. Neither calibration rule controls the shift term $\delta_k$. Comparing the two bounds does not rank actual coverage errors, and the bounds concern expected coverage: if $C_k$ denotes coverage conditional on the realized calibration data, the proposition bounds $|\E C_k-(1-\alpha)|$, not $\E|C_k-(1-\alpha)|$. Appendix~\ref{app:bands} gives simultaneous bounds for realized coverage.

With two groups, write $p=\Prb_{\rm cal}(G=1)$ and $d_F=\sup_t|F_1(t)-F_0(t)|$. Since $F=(1-p)F_0+pF_1$, we have $\eta_0=pd_F$ and $\eta_1=(1-p)d_F$. Different missingness rates therefore need not imply a large pooling term: if the mask is independent of the score, $\eta_0=\eta_1=0$ whatever the rates. The selector of Section~\ref{sec:theory} targets variation in missingness rates across sites, which is neither necessary nor sufficient for large $\eta_k$.

Because thresholds use inclusive comparisons ($S\le q$), ties do not weaken the lower bound, but they can produce overcoverage, which the atom terms capture. Tree ensembles produce such ties, so we retain atom terms rather than assume continuous scores. The fraction of repeated scores in a sample is not a reliable estimate of the largest population atom.

\subsection{Proofs}

\begin{lemma}
Let $S_1,\ldots,S_N$ be iid with CDF $F$ and largest atom $a$. For $\tau=1-\alpha$, let $r=\lceil(N+1)\tau\rceil$ and let $q=S_{(r)}$ if $r\le N$, otherwise $q=+\infty$. For $N\ge1$,
\begin{equation}
 \tau\le \E F(q)\le\tau+\frac{1}{N+1}+a.
\end{equation}
\end{lemma}
\begin{proof}
Introduce independent $V_i\sim\mathrm{Uniform}(0,1)$ solely for the proof and set
\[
 W_i=F(S_i-)+V_i\{F(S_i)-F(S_i-)\},
\]
where $F(S_i-)$ is the left limit of $F$ at $S_i$. This randomized probability integral transform makes the $W_i$ iid uniform on $(0,1)$. It preserves the order of unequal scores and spreads tied scores across the corresponding jump of $F$, whose height is at most $a$. Consequently, for $r\le N$,
\[
 W_{(r)}\le F(S_{(r)})\le W_{(r)}+a.
\]
Taking expectations and using $\E W_{(r)}=r/(N+1)$ proves the result, since $\tau\le r/(N+1)<\tau+1/(N+1)$. If $r>N$, then $F(q)=1$ and $\alpha<1/(N+1)$, so the bounds still hold. The implemented procedure uses inclusive comparisons and no randomization; the uniforms are only a proof device.
\end{proof}

\begin{proof}[Proof of Proposition~\ref{prop:coverage}]
For pooled calibration with realized threshold $q$, coverage in group $k$ at deployment is $H_k(q)$. Pointwise,
\[
 |H_k(q)-F(q)|\le |H_k(q)-F_k(q)|+|F_k(q)-F(q)|\le\delta_k+\eta_k.
\]
Averaging over calibration data, applying the lemma to $F$, and using the triangle inequality gives \eqref{eq:pool}. For groupwise calibration, condition on the calibration group labels $G_{1:n}$. Because $G$ is fixed and the observations are iid, the $n_k$ scores in group $k$ are then iid from $F_k$, while the independent deployment score has CDF $H_k$. Applying the lemma to $F_k$ and bounding $|H_k-F_k|$ by $\delta_k$ gives \eqref{eq:group}. An empty group receives the full prediction set and coverage one; we treat this case separately rather than as a calibrated estimate.
\end{proof}

If a group's score distribution does not change at deployment, then $\delta_k=0$ and the group attains the usual expected coverage bound, including terms for sample size and ties. Under arbitrary shift, $\delta_k$ can make the bound uninformative.

\subsection{Simultaneous bounds for realized group coverage}
\label{app:bands}

Fix $0<\beta<1$ and condition on the calibration group labels and all fitted objects, with $K$ fixed groups. The $r$th order statistic of $N$ iid uniforms has a $\operatorname{Beta}(r,N+1-r)$ distribution. For every group with $1\le r_k\le n_k$, let $b_k^-$ and $b_k^+$ be the $\beta/(2K)$ and $1-\beta/(2K)$ quantiles of $\operatorname{Beta}(r_k,n_k+1-r_k)$. Then, with probability at least $1-\beta$ over calibration scores, simultaneously for these groups,
\begin{equation}
 \max\{0,b_k^--\delta_k\}\le C_k\le
 \min\{1,b_k^++a_k+\delta_k\},\qquad C_k=H_k(q_k).
 \label{eq:bands}
\end{equation}
Groups with $q_k=+\infty$, including empty groups, have $C_k=1$ and need no interval.

To prove this, apply the lemma's transform within group $k$: its $r_k$th uniform order statistic lies in $[b_k^-,b_k^+]$ except with probability $\beta/K$. On this event, $b_k^-\le F_k(q_k)\le b_k^++a_k$, and the shift bound adds $\delta_k$ on each side. A union bound over at most $K$ groups proves \eqref{eq:bands}; independence across groups is not needed. Although auxiliary uniforms establish the event, the intervals depend only on calibration scores, so the failure probability remains bounded after integrating out the uniforms.

For continuous scores without shift, $C_k$ itself has this beta distribution, with variance
\[
 \operatorname{Var}(C_k)=
 \frac{r_k(n_k+1-r_k)}{(n_k+1)^2(n_k+2)}.
\]
At target 0.9 and $n_k=1000$, its standard deviation is about 0.0095, whereas the bias term $1/(n_k+1)$ is about 0.0010; sampling variability in realized coverage therefore dominates. The simultaneous intervals also bound the largest absolute deviation over the calibrated groups by their furthest endpoint from the target. They do not cover arbitrary overlapping audit groups, finite-test estimation error, or dependence within hospitals, and the unknown shift and atom terms can make them uninformative. This is a standard order-statistic consequence, not a new guarantee specific to missingness.

\section{Pooled and within-site aggregation}
\label{app:hospital-mechanism}

We use the notation of Section~\ref{sec:aggregation}. For a calibration method and fixed site weights $w$, $A$ averages signed coverage errors across sites before taking the largest absolute value over groups, $B$ takes each site's largest absolute error before averaging, and $D$ lies between them. By Proposition~\ref{prop:aggregation}(iii), $A\le D\le B$, and $B-A=(D-A)+(B-D)$ separates cancellation of signed errors from worst-group switching across sites.

\subsection{Population example}

To separate these effects from sampling noise, we use a population with known score distributions. There are two hospitals $h\in\{0,1\}$ and two mask states $g\in\{0,1\}$. In the calibration population, the four cells have equal probability, and the true-label score in cell $(h,g)$ is uniform on $[0,u_{hg}]$, with $u=\left(\begin{smallmatrix}1&0.5\\0.7&1.1\end{smallmatrix}\right)$. Pooled calibration uses the 0.9 quantile of the mixture of all four cells (0.838); mask calibration uses the 0.9 quantile within each mask state, mixing over hospitals (0.80 for $g=0$ and 0.88 for $g=1$). Using population quantiles removes calibration sampling error. These generic scores are not generated by the binary mortality model, and the example is not a model of the clinical cohorts.

At deployment, the hospitals receive weights $(w_0,1-w_0)$ with $w_0\in\{0.5,0.8\}$, and the mask states remain equally likely within each hospital. Independently, the score distribution in cell $(1,1)$ is either unchanged or shifted upward by 0.12, becoming uniform on $[0.12,1.22]$; the other cells are unchanged. The setting $w_0=0.5$ without shift reproduces the calibration population. Coverage in each cell is computed exactly from the uniform CDFs, and $A$, $D$, and $B$ use the deployment weights.

\begin{table}[ht]
\caption{Exact population aggregation quantities at nominal coverage 0.9. $w_0$ is the first hospital's deployment weight. Shift applies only to the second hospital's second mask state. These are population values, not estimates from simulated samples.}
\label{tab:hospital-population}
\centering\footnotesize\setlength{\tabcolsep}{3pt}
\input{tables/hospital_table.tex}
\end{table}

Without any deployment change, mask calibration attains exact coverage in both mask states, so $A=0$. Yet hospital 0 has coverage 0.80 in state 0 and 1.00 in state 1, and hospital 1 the reverse, so each hospital has a gap of 0.1 and $B=0.1$: exact group coverage coexists with errors in every hospital. Changing only the hospital weights to $w_0=0.8$ makes $A$ favor pooling while $B$ favors mask calibration. The within-cell shift changes coverage but not these rankings. Since $B-D=0$ throughout, the nonlinear effect in this example is cancellation, not worst-group switching.

Finite samples can obscure these comparisons. We also draw 25, 100, or 1,000 test observations per cell in 1,000 repetitions, evaluating both methods on the same draws. In the no-shift, equal-weight reference, the population difference in $A$ (mask minus pooling) is $-0.0190$. With 25 observations per cell, the empirical difference has mean $-0.0038$ (SD $0.0215$), and its sign is wrong or tied in 64.2\% of repetitions; with 1,000 observations per cell, this fraction falls to 1.9\%. The mean is pulled toward zero because the empirical $A$ is the absolute value of a noisy average and is biased upward, most strongly for mask calibration, whose population value is zero. These results demonstrate possible aggregation and sampling effects, not the causes of the clinical results in eICU.

\subsection{Clinical decomposition}
\label{app:clinical-decomposition}

The clinical decomposition in Section~\ref{sec:experiments} is computed within each eICU assignment over its test hospitals and then averaged across assignments. Patient weights are proportional to each hospital's number of test patients, $w_h=n_h/n$; equal weights give each hospital weight $1/H$. The decomposition excludes a hospital when either selected mask state is absent, so both methods use the same hospitals. This differs from the within-hospital comparison in Section~\ref{sec:experiments}, which first averages each unique hospital's repeated evaluations and, for a hospital lacking one state, uses the available state; an empty cell is treated as undefined, never as zero coverage. Across the ten eICU assignments, there are nine empty hospital--state cells and 22 cells with fewer than 40 patients, so small denominators can add substantial noise.

Table~\ref{tab:clinical-decomposition} reports paired differences, with $\Delta$ denoting missingness minus pooling. The contributions add to $\Delta B-\Delta A$ and describe arithmetic differences, not causal hospital processes. With patient weights, cancellation contributes $+0.01135$ for LR and $+0.00631$ for MLP, and worst-group switching contributes $-0.00076$ and $+0.00071$; neither comparison changes sign. For XGBoost, both nonlinear contributions are near zero ($+0.00011$ and $+0.00002$), and the sign change occurs in the weighting step. Equal weights give the same pattern.

\begin{table}[ht]
\caption{Paired missingness-minus-pooling differences in the clinical hospital decomposition. Means over ten eICU assignments, with the same complete hospitals and weights for both methods.}
\label{tab:clinical-decomposition}
\centering\footnotesize\setlength{\tabcolsep}{4pt}
\input{tables/clinical_decomposition.tex}
\end{table}

\paragraph{Covariance in the XGBoost weighting reversal.}
For each eICU assignment, we compute the covariance $\operatorname{Cov}_w(p_{hg}/\bar p_g,e_{hg})$ of Proposition~\ref{prop:aggregation}(ii) for both selected mask states using patient weights. Table~\ref{tab:covariance} reports signed means across assignments. Adding each covariance to the corresponding common-weight signed error reproduces the pooled signed error exactly. After taking absolute maxima within each assignment, the mean weighting term is $\Delta A_w-\Delta\gap_{\rm pool}=+0.00708$, which moves the comparison from $-0.00515$ to $+0.00192$. Because the maximum is nonlinear, this weighting term is not the difference between the two mean covariances.

\begin{table}[ht]
\caption{Mean signed covariance terms for eICU XGBoost over ten assignments. State 0 is observed and state 1 is missing. Difference is missingness minus pooled calibration; entries are in coverage-probability units.}
\label{tab:covariance}
\centering\footnotesize
\input{tables/covariance_table.tex}
\end{table}

\section{Implementation details}
\label{app:implementation}

\paragraph{Preprocessing.}
We exclude identifiers, outcomes, fields derived at discharge, and supplied mortality probabilities. Features that are constant or entirely missing in the training data are removed. Categorical variables are coded as indicator (one-hot) variables with levels learned from training data; missing and unseen categories share a dedicated indicator. XGBoost handles missing numeric values internally; for LR and the MLP, they are replaced with zero before standardization with training-data means and standard deviations. Group indicators are computed from the raw records before this imputation, and overall missingness is the fraction of raw predictors that are missing.

\paragraph{Predictors.}
Models use scikit-learn \citep{Pedregosa2011} and XGBoost \citep{Chen2016}. LR weights each class inversely to its frequency and is fitted with the liblinear solver for at most 1,000 iterations. XGBoost uses 64 trees of maximum depth three, learning rate 0.08, and random subsampling of 90\% of rows and columns. The MLP has one hidden layer of 64 units with activation $\max(0,x)$ and an $L_2$ penalty of $10^{-4}$, and is fitted with the Adam stochastic-gradient optimizer at learning rate 0.001. Training stops when the training loss fails to improve by $10^{-4}$ for ten consecutive epochs (passes through the data), with a cap of 2,000 epochs; all fits stop between 132 and 394 epochs. No hyperparameter or stopping rule is selected using test performance. Mean test AUROCs for LR, XGBoost, and MLP are 0.879, 0.884, and 0.788 in eICU and 0.866, 0.873, and 0.783 in MIMIC-IV. The MLP discriminates less well than the other predictors, but the principal findings do not rely on it: XGBoost supplies the weighting reversal, and LR also shows substantial attenuation through cancellation (Table~\ref{tab:clinical-decomposition}).

\paragraph{Sensitivity to the training cap.}
With a cap of 200 epochs, all ten eICU and two MIMIC-IV MLP fits stop before meeting the stopping rule. Raising the cap, with architecture, seeds, data order, and preprocessing unchanged, lets all 12 fits meet the rule within 225--394 epochs; all reported MLP results use these converged fits. Under the 200-epoch cap, the mean within-hospital MLP difference is $+0.0024$ rather than $-0.0004$, so the sign of this small difference is sensitive to optimization. Meeting the stopping rule does not establish a global minimum.

\paragraph{Selection and controls.}
eICU candidates include measurements from the first hour and first day, plus physiological variables from the Acute Physiology and Chronic Health Evaluation (APACHE) severity score. MIMIC-IV candidates include first-day minimum, maximum, and mean measurements and urine output. Institutional descriptors are excluded, and eligibility thresholds are fixed before testing. Ties in \eqref{eq:select} are broken by sorting variable names. The selector chooses urine output in five eICU assignments and maximum bicarbonate in the other five; MIMIC-IV choices span six measurement summaries. The random control draws one eligible variable uniformly in each assignment, and the clinical control uses availability of the first-day maximum lactate. Risk groups use the $1/3$ and $2/3$ quantiles of predicted risk in the selection sample, with ties assigned to the higher group. The score tree is a regression tree with at most three leaves, each containing at least 100 selection observations.

\paragraph{Mask weighting.}
For selected mask state $g$, let $N_g$ and $T_g$ be the calibration and target counts, with totals $N$ and $T$. The weight
\[
 v_g=\frac{(T_g+1)/(T+2)}{(N_g+1)/(N+2)}
\]
is a smoothed estimate of the ratio of target to calibration probabilities of state $g$. Calibration score $S_i$ receives weight $v_{G_i}$. For a test record in state $g$, the threshold is the smallest $t$ such that $\sum_i v_{G_i}\mathbf 1\{S_i\le t\}\ge(1-\alpha)(\sum_i v_{G_i}+v_g)$; equivalently, the test record's own weight is placed at $+\infty$, as in weighted conformal prediction \citep{Tibshirani2019}. This does not remove error in the estimated ratios or establish the assumptions needed under general covariate shift.

\paragraph{Conditional conformal comparator.}
The method of \citet{Gibbs2025} fits an augmented quantile regression of the score on a chosen function class and guarantees coverage under every covariate shift whose density ratio is a nonnegative function in that class. We use the \texttt{conditionalconformal} package (v0.0.5) with the linear span of a constant, the fitted probability, and overall missingness, in its deterministic (non-randomized) prediction mode at level 0.9, which is conservative. Calibration data and fitted predictions are shared with the other methods, and the 2,000 test patients per assignment are drawn without replacement.

\paragraph{Uncertainty.}
Table~\ref{tab:core} reports means and paired intervals for the pooled selected-gap improvement (Appendix~\ref{app:paired}). Figure~\ref{fig:paired} shows standard deviations across assignments, not confidence intervals. For the within-hospital comparison, we first average each hospital's repeated paired gap differences and then resample hospitals with replacement 10,000 times; six-comparison adjusted intervals use percentile levels $0.05/12$ and $1-0.05/12$. Training and calibration sites are held fixed, so these summaries do not include refitting uncertainty. All random seeds are fixed and provided with the code.

\subsection{Paired uncertainty for the headline improvement}
\label{app:paired}
The estimand is the mean across assignments of the pooled selected gap under pooled calibration minus that under missingness calibration. For eICU, each bootstrap replicate samples the 76 distinct test hospitals with replacement. A hospital receives the same multiplicity in every assignment and for both methods. We multiply its covered counts and denominators by this multiplicity, recompute each assignment's two pooled group coverages, take the maximum absolute deviation from 0.9, and only then average the paired improvements. This preserves pairing and repeated appearances without treating assignments as independent. One of 10,000 draws omitted all observations of a selected group in an assignment and was excluded.

Each MIMIC-IV assignment tests one unit. We therefore average the three paired improvements for each unit and resample the six unit averages. Let $\widehat\theta$ be the observed mean improvement and $\xi_u$ the $u$ quantile of its bootstrap replicates $\theta^*$. The 95\% basic interval is $[2\widehat\theta-\xi_{0.975},\,2\widehat\theta-\xi_{0.025}]$, and adjusted endpoints use $\xi_{1-0.05/12}$ and $\xi_{0.05/12}$ in the same way (Table~\ref{tab:paired-full}). Because the maximum is nonlinear, the bootstrap distribution can be shifted relative to $\widehat\theta$; the basic interval uses $\theta^*-\widehat\theta$ as a pivot and so accounts for this shift, but it does not remove the finite-sample bias of the maximum. These intervals are conditional, descriptive approximations: they neither refit predictors nor regenerate calibration samples, and six units cannot support strong inference about a wider population of hospitals.

\begin{table}[ht]
\caption{Paired improvement in pooled selected gap (pooled minus missingness), in percentage points. Basic intervals reflect bootstrap quantiles around the observed improvement. Adjusted endpoints use quantiles at $0.05/12$ and $1-0.05/12$ for six comparisons.}
\label{tab:paired-full}
\centering\footnotesize\setlength{\tabcolsep}{3pt}
\input{tables/paired_intervals.tex}
\end{table}

\subsection{Sensitivity to evaluation-group size}
\label{app:panel-size}
We retain the raw panel gap as the primary metric. To check whether its comparison is driven by small groups, we repeat it after requiring each evaluated group to contain at least 500 or 1,000 test patients. Each cutoff selects the same groups for every method within an assignment and is not tuned on performance; calibration is unchanged.

We also report $\max_g |\widehat C_g-0.9|/\sqrt{0.9(0.1)/n_g}$ over nonempty panel groups, which divides each deviation by its nominal binomial standard error. This standardized error is a descriptive scale check, not a corrected coverage gap, a significance test, or a simultaneous confidence bound; it ignores within-hospital dependence and gives larger groups more weight. All six mean comparisons continue to favor risk terciles under both cutoffs and on this alternative scale (Table~\ref{tab:panel-size}).

\begin{table}[ht]
\caption{Missingness minus risk-tercile panel error. Positive entries favor risk calibration. The first three columns are raw absolute gaps with the stated minimum group size; the last uses the nominal standardized error, in different units. Means across assignments.}
\label{tab:panel-size}
\centering\footnotesize\setlength{\tabcolsep}{4pt}
\input{tables/panel_sensitivity.tex}
\end{table}

\clearpage
\section{Additional results}
\label{app:tables}

All analyses reuse the 84 fitted predictors without test-based tuning. The table below reports overall coverage, average set size, selected gap, and panel gap for every calibration rule. The score-tree and mask-weighting rows are our stated controls, not implementations of other published methods for missing data. Risk (2) splits selection-sample risk at its median and applies the same conformal quantile within each group; it matches the number of thresholds of missingness calibration, not its group sizes. Risk + label is the joint rule of Appendix~\ref{app:outcome-detail}. Appendix~\ref{app:pilot} reports an exploratory extension with learned thresholds.
\begingroup\footnotesize\setlength{\tabcolsep}{4pt}
\input{tables/all_methods.tex}
\endgroup

\subsection{Coverage diagnostics}
\label{app:diagnostics}

The maximum panel gap is an empirical diagnostic, not a precise estimate of the largest population error: selecting the worst of many overlapping groups inflates it and adds uncertainty. We do not give patient-level iid confidence intervals for these clustered clinical data, or claim simultaneous empirical validity from the iid bound of Appendix~\ref{app:bands}.

\paragraph{Which patients receive different sets?}
\label{app:set-difference}
We compare the label-inclusion decisions of pooled and missingness calibration on identical test patients; a patient counts as changed if either label's inclusion differs. Table~\ref{tab:set-difference} averages this fraction across assignments; its range describes variation across assignments, not uncertainty. For eICU XGBoost, the mean is 1.03\% (range 0.01--3.91\%), or 1.04\% when assignments are weighted by test size. Most sets are therefore unchanged, although group-specific coverage can differ.

\begin{table}[ht]
\caption{Test patients whose prediction sets differ between missingness and pooled calibration. Means and ranges are across ten eICU or 18 MIMIC-IV assignments.}
\label{tab:set-difference}
\centering\footnotesize
\input{tables/set_disagreement.tex}
\end{table}

\paragraph{Direction and outcome coverage.}
In the following table, $U$ and $O$ are computed over the two selected mask states, and the panel gap over all eligible measurements. Entries are means across assignments, not estimates for a new hospital.
\begingroup\footnotesize\setlength{\tabcolsep}{3pt}
\input{tables/diagnostic_table.tex}
\endgroup

\paragraph{Set composition.}
The four columns give the fractions of test patients receiving each type of set and sum to one before rounding. These proportions describe prediction behavior, not a clinical action policy.
\begingroup\footnotesize\setlength{\tabcolsep}{4pt}
\input{tables/composition_table.tex}
\endgroup

\paragraph{Conditional conformal comparison.}
Table~\ref{tab:aligned} reports the 2,000-patient comparison of Section~\ref{sec:experiments}. The extended basis $(1,f(Z),\overline M,M_{j^*})$ adds the selected mask indicator $M_{j^*}$ to the original basis, where $\overline M$ is overall missingness; all other settings and test patients are identical. This tests one prespecified change, not an optimized function class, and its results should not be compared with those on the full test samples.
\begin{table}[ht]
\caption{Conditional conformal prediction with and without the selected mask in its basis, XGBoost, identical 2,000-patient test samples. Gaps and shortfall use nominal coverage 0.9.}
\label{tab:aligned}
\centering\footnotesize\setlength{\tabcolsep}{4pt}
\input{tables/conditional_table.tex}
\end{table}

\clearpage
\section{Outcome coverage and score transformations}
\label{app:outcome-detail}

The analyses in this section use the original assignments and fitted predictors. The error allowances $\alpha\in\{0.05,0.1,0.2\}$, odds multipliers $\{1/4,1,4\}$, and joint calibration rule were fixed before these analyses, which were motivated by the preceding results and are exploratory; no other settings were tried. Overall and outcome-specific coverage are computed within each test population before averaging, and empty populations or labels are treated as undefined rather than as zero coverage.

\paragraph{Direction of errors.}
Absolute gaps do not distinguish undercoverage from conservatism. For eICU XGBoost, missingness and pooled calibration have the same mean selected shortfall $U$ to four decimals, 0.0169, but mean excesses $O$ of 0.0092 and 0.0159. In this setting, the smaller absolute gap reflects less overcoverage, not less undercoverage (Appendix~\ref{app:diagnostics}).

\paragraph{Death coverage.}
Let $C_y=\Prb(y\in\Gamma\mid Y=y)$ be coverage among patients with outcome $y$, where sets containing both labels count as covering, and let $\mu=\Prb(Y=1)>0$ be mortality in a fixed population. Since $C=(1-\mu)C_0+\mu C_1$ and $C_0\le1$,
\begin{equation}
 C_1\ge\max\{0,1-(1-C)/\mu\}.
\end{equation}
If $C\ge1-\alpha$, this gives $C_1\ge1-\alpha/\mu$, which is vacuous when $\mu\le\alpha$ even without shift. Mortality across eICU test assignments ranges from 0.081 to 0.109, and seven of ten assignments are below 0.1. Near-target overall coverage can therefore coexist with low death coverage. Tables~\ref{tab:prevalence} and~\ref{tab:population-outcomes} report denominators, mortality, and outcome coverage within the selected mask states.

\begin{table}[ht]
\caption{Population denominators and mortality ranges, including repeated assignment-specific cells. Mortality ranges exclude empty cells; denominator ranges retain them. $N_{\rm empty}$ counts cells with zero patients. Mask refers to the two selected states. Counts are shared by all predictors and include repeated evaluation of a hospital across assignments.}
\label{tab:prevalence}
\centering\footnotesize\setlength{\tabcolsep}{3pt}
\input{tables/prevalence_table.tex}
\end{table}

\begin{table}[ht]
\caption{Outcome coverage within the selected mask states at $\alpha=0.1$, missingness calibration. Unweighted means over nonempty cells; $C_y$ excludes cells without label $y$. Bound averages $\max\{0,1-(1-C)/\mu\}$ over cells with deaths, using each cell's actual coverage and mortality; it is descriptive, not a guarantee for each cell.}
\label{tab:population-outcomes}
\centering\footnotesize\setlength{\tabcolsep}{4pt}
\input{tables/population_outcomes.tex}
\end{table}

With missingness calibration, mean eICU XGBoost death coverage is 0.470, 0.161, and 0.044 at $\alpha=0.05$, 0.1, and 0.2, while overall coverage is 0.949, 0.895, and 0.795, and survivor coverage at $\alpha=0.1$ is 0.970 (Table~\ref{tab:alpha}). Low death coverage is therefore not confined to settings where the error allowance exceeds mortality. Predictors also differ: at $\alpha=0.1$, death coverage is 0.885 for LR and 0.448 for MLP. LR uses class weighting and XGBoost does not; similar ranking accuracy does not imply similar probability scales.

\begin{table}[ht]
\caption{Missingness calibration across coverage levels, recalibrated on the original calibration patients. $U$ is the largest shortfall below $1-\alpha$ across the measurement panel. Means across the ten eICU or 18 MIMIC-IV assignments.}
\label{tab:alpha}
\centering\footnotesize\setlength{\tabcolsep}{4pt}
\input{tables/alpha_table.tex}
\end{table}

\paragraph{Which probability transformations change the sets?}
\label{sec:invariance}
Risk ranking and score ranking are different objects. Writing $r=f(Z)$, we have $S(r,0)=r$ and $S(r,1)=1-r$, so calibration that pools both outcomes orders the true-label scores of survivors and deaths together. A transformation of the predicted probabilities can preserve the risk ranking of patients while changing this joint order of scores.

\begin{proposition}[Set invariance within fixed calibration cells]
\label{prop:invariance}
Fix all calibration and test cell memberships, the calibration observations, and the quantile ranks. Applying the same strictly increasing function $\phi$ to every calibration score and every candidate score compared within a cell leaves all label-inclusion decisions unchanged. This holds with tied scores under inclusive thresholds; cells assigned $+\infty$ continue to include every candidate.
\end{proposition}
\begin{proof}
For a finite threshold at rank $\ell$, $q=S_{(\ell)}$. Because $\phi$ is strictly increasing, the transformed threshold is $\phi(S_{(\ell)})=\phi(q)$, and $\phi(s)\le\phi(q)$ if and only if $s\le q$. Infinite thresholds are unchanged.
\end{proof}

Let $\operatorname{logit}(r)=\log\{r/(1-r)\}$ and $\operatorname{expit}$ be its inverse. Consider the odds shift $T_b(r)=\operatorname{expit}\{\operatorname{logit}(r)+b\}$, which multiplies the odds by $e^b$ and is equivalent to shifting the intercept of a logistic model, and the temperature change $\tilde T_\kappa(r)=\operatorname{expit}\{\kappa\operatorname{logit}(r)\}$ with $\kappa>0$, which rescales the log-odds; both preserve the risk ranking and the endpoints 0 and 1. Under the odds shift, $S_b(r,0)=T_b(S(r,0))$ but $S_b(r,1)=T_{-b}(S(r,1))$: the two labels are transformed by different functions, so their joint score order, and hence pooled, missingness-aware, and risk-only sets, can change although the risk ranking does not. In contrast, $\tilde T_\kappa(1-r)=1-\tilde T_\kappa(r)$, so $S_\kappa(r,y)=\tilde T_\kappa(S(r,y))$ for both labels, and every set is unchanged by Proposition~\ref{prop:invariance}. Label-conditional calibration, which fits separate thresholds for survivors and deaths, and joint risk-by-label calibration are invariant to odds shifts as well, because within each cell the score and its threshold undergo the same map, $T_b$ or $T_{-b}$. These identities hold in exact arithmetic; in floating point, subtraction from one and evaluation of the logistic function can merge distinct scores near 0 and 1 and create threshold ties.

\paragraph{Probability-scale sensitivity.}
We multiply the predicted odds by $1/4$, 1, and 4 ($b\in\{-\log4,0,\log4\}$), apply the same transformation to calibration and test probabilities, and recalibrate the same missingness groups at $\alpha=0.1$. These transformations preserve the risk ranking and hence the AUROC. On eICU XGBoost, death coverage becomes 0.024, 0.161, and 0.593, while overall coverage stays within 0.894--0.896 (Figure~\ref{fig:outcomes} and Table~\ref{tab:scale}). No transformation is selected using test outcomes. This sensitivity is specific to odds shifts and does not estimate what retraining with class weights would produce.

\begin{figure}[t]
\centering
\includegraphics[width=.86\linewidth]{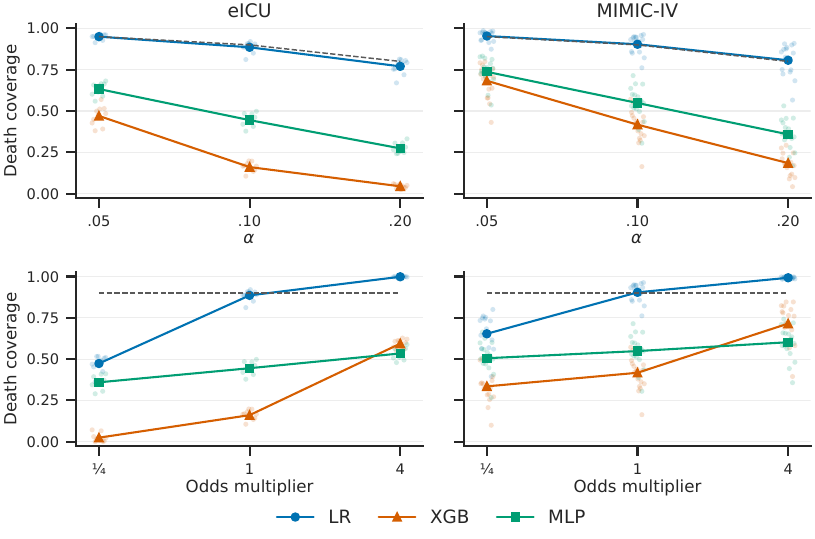}
\caption{Death coverage under missingness-aware calibration. Top: varying the error allowance $\alpha$. Bottom: multiplying predicted odds at fixed $\alpha=0.1$ and recalibrating; rankings and AUROC are unchanged. Faint points are individual assignments (ten eICU, 18 MIMIC-IV); lines show means; dashed lines mark nominal coverage. Assignments overlap, so no independence-based interval is implied.}
\label{fig:outcomes}
\end{figure}

\begin{table}[ht]
\caption{Multiplying predicted odds and recalibrating missingness groups at $\alpha=0.1$. Overall and outcome coverage, mean set size, and panel shortfall are assignment means.}
\label{tab:scale}
\centering\footnotesize\setlength{\tabcolsep}{4pt}
\input{tables/scale_table.tex}
\end{table}

\paragraph{Calibrating risk and outcome jointly.}
\label{sec:alignment}
These results suggest calibrating risk groups and outcomes together. Joint calibration fits six thresholds, one for each combination of risk tercile $g$ (fixed by the selection sample) and candidate outcome $y$: the rank-$\lceil(n_{gy}+1)(1-\alpha)\rceil$ order statistic of the scores $1-\widehat p_y$ among the $n_{gy}$ calibration observations with tercile $g$ and outcome $y$, or $+\infty$ if the rank exceeds $n_{gy}$. At prediction, each candidate label is compared with its own cell threshold, so the true test label is never used. The smallest calibration cells, across predictors and assignments, contain 16 eICU and 11 MIMIC-IV patients. No cell requires $+\infty$ at $\alpha=0.1$ or 0.2 (at $\alpha=0.05$, three of 180 eICU cells and two of 324 MIMIC-IV cells do), but small cells select their maximum score and can be conservative. Cellwise validity requires exchangeability within cells and does not follow under arbitrary shift.

For XGBoost, joint calibration achieves mean death coverage of 0.908 on eICU and 0.906 on MIMIC-IV (Table~\ref{tab:joint-all}). The largest changes occur where risk and outcome intersect (Table~\ref{tab:six-cells}). Death coverage in the lowest risk tercile rises from 0.000 and 0.015 under label-only calibration to 0.916 and 0.892 under joint calibration, and coverage among high-risk survivors rises from 0.638 to 0.895 in eICU and from 0.442 to 0.901 in MIMIC-IV. Across assignments, joint coverage of low-risk deaths ranges from 0.828 to 1.000 in eICU and from 0.750 to 1.000 in MIMIC-IV; Table~\ref{tab:target-counts} gives the cell counts.

The cost is ambiguity: for XGBoost, joint calibration returns both labels for 63.7\% of eICU and 61.1\% of MIMIC-IV patients on average, with no empty sets (Table~\ref{tab:joint-composition}). Mean set size equals $1-P(\varnothing)+P(\{0,1\})$, so without empty sets the double-label fraction equals mean size minus one. For XGBoost, joint calibration reduces the panel shortfall relative to label-only calibration but not relative to risk-only calibration. Small calibration cells may contribute to this conservatism even when thresholds are finite; more efficient ways of targeting the same intersections may exist.

\begin{table}[ht]
\caption{Coverage in the six risk--outcome cells for XGBoost. Risk tercile increases from 1 to 3; labels 0 and 1 denote survival and death. Means and joint-rule ranges are over ten eICU or 18 MIMIC-IV assignments; ranges describe variation across assignments and are not confidence intervals.}
\label{tab:six-cells}
\centering\footnotesize\setlength{\tabcolsep}{4pt}
\input{tables/six_cells.tex}
\end{table}

\begin{table}[ht]
\caption{Single-target and joint calibration at $\alpha=0.1$, full test samples. $C$ is overall coverage, $C_1$ death coverage, and $U$ panel shortfall.}
\label{tab:joint-all}
\centering\footnotesize\setlength{\tabcolsep}{4pt}
\input{tables/joint_table.tex}
\end{table}

\begin{table}[ht]
\caption{Joint calibration denominators, quantile ranks, and test denominators, XGBoost. Entries are ranges across assignments. Max rank counts assignments in which the quantile rank equals the calibration count and therefore selects the maximum score. No test cell is empty.}
\label{tab:target-counts}
\centering\footnotesize\setlength{\tabcolsep}{4pt}
\input{tables/cell_counts.tex}
\end{table}

\begin{table}[ht]
\caption{Four output-set types for the target-alignment comparisons, computed within each assignment and then averaged. Entries are mean fractions of full test populations at $\alpha=0.1$.}
\label{tab:joint-composition}
\centering\footnotesize\setlength{\tabcolsep}{4pt}
\input{tables/target_composition.tex}
\end{table}

\clearpage
\section{Exploratory calibration beyond risk}
\label{app:pilot}

This exploratory analysis asks whether missingness carries information about the score quantile beyond predicted risk. It was specified after inspecting the preceding clinical comparisons, and its settings were fixed for all 28 XGBoost assignments without further search. It uses the same selection, calibration, and full test samples and the same fitted predictors, with true-label score $S=1-\widehat p_Y$. Unlike the partition comparisons, it learns a threshold as a function of the inputs and then calibrates the residual scores, following the general principle of calibrating a fitted quantile \citep{Romano2019}; we do not claim a new conformal construction.

\paragraph{Threshold learning and validation.}
Let $r=f(Z)$ denote predicted mortality risk and $M_{\mathcal J}$ the eligible measurement masks. On selection data, we estimate the conditional 0.9 quantile of $S$ given $r$ alone, $q_r(r)$, and given $r$ and the masks, $q_m(r,M_{\mathcal J})$. Both use quantile (pinball-loss) regression with histogram-based gradient-boosted trees: 40 iterations, four leaves, minimum leaf size 100, learning rate 0.1, $L_2$ regularization 10, and no early stopping. These are matched model settings, not matched effective complexity.

For validation, sorted selection-site identifiers are assigned alternately to two folds, and each learner predicts the fold excluded from its fit. For $\lambda\in\{0,0.5,1\}$, define $q_\lambda=(1-\lambda)q_r+\lambda q_m$. We choose $\lambda$ to minimize the largest site-mean pinball loss $\rho_\tau(S-q_\lambda)$, where $\rho_\tau(x)=\max\{\tau x,(\tau-1)x\}$ and $\tau=0.9$, with ties favoring smaller $\lambda$. We then refit both learners on all selection data. MIMIC-IV has only two selection units per assignment, which limits this validation. Pinball loss is a proxy and does not guarantee lower coverage error at a new site.

\paragraph{Residual calibration.}
For each of $q_r,q_m,q_\lambda$, we compute calibration residuals $R_i=S_i-q(Z_i)$ and their conformal quantile $\widehat t$ at rank $\lceil(n+1)0.9\rceil$. The candidate label $y$ is included when $1-\widehat p_y(Z)\le q(Z)+\widehat t$; thresholds are not clipped. Given an independently fitted $q$, split-conformal marginal validity requires exchangeability of calibration and test residuals, which does not establish group or hospital validity under the shifts studied here. Label-conditional calibration instead computes the original score quantile separately for each calibration label and evaluates each candidate against its own quantile; its usual guarantee likewise requires exchangeability within labels.

\begin{table}[ht]
\caption{Learned threshold comparison, assignment means. $U$ is the largest undercoverage on the full measurement panel; Site $U$ first computes this maximum within each nonempty test site and then averages sites within an assignment. $C_1$ is coverage among deaths. Missingness and Risk bins are the original missingness and risk-tercile methods. All rows use identical test patients within an assignment.}
\label{tab:pilot}
\centering\footnotesize\setlength{\tabcolsep}{4pt}
\input{tables/residual_calibration.tex}
\end{table}

The validation step selects no mask correction in 21 of 28 assignments (nine of ten in eICU and 12 of 18 in MIMIC-IV) and $\lambda=1$ in the remaining seven. Even when selected in eICU, the correction does not change the final sets. On MIMIC-IV, site and pooled shortfalls coincide because each assignment tests one unit, whereas sparse within-site groups can make eICU site maxima noisy. We retain all assignments and do not interpret this limited family as evidence that no better missingness method exists.

\paragraph{What the binary quantile can detect.}
Let $\mu(r,m)=\Prb(Y=1\mid f(Z)=r,M=m)$ be the conditional mortality. For $r<1/2$, the true-label score is $r$ with probability $1-\mu(r,m)$ and $1-r$ otherwise, so its lower 0.9 quantile is $r$ when $\mu(r,m)\le0.1$ and $1-r$ when $\mu(r,m)>0.1$. For $r>1/2$, the lower value is $1-r$, which is the quantile when $\mu(r,m)\ge0.9$; otherwise the quantile is $r$. At $r=1/2$ both scores equal $1/2$. The 0.9 quantile thus changes with the mask only when $\mu(r,m)$ crosses 0.1 or 0.9, so additional outcome information can leave it unchanged. Moreover, $q_m=q_r+c$ for a constant $c$ yields identical residual-calibrated thresholds, because subtracting $c$ from every calibration residual shifts its order statistic by $-c$. Even nonconstant threshold changes need not cross either candidate score.

\paragraph{Controlled detection examples.}
We fix predicted risk at 0.2 and generate a balanced binary mask. Conditional death probabilities in the two mask states are $(0.08,0.08)$, $(0.02,0.08)$, or $(0.02,0.25)$, representing no additional information, information that leaves the 0.9 quantile unchanged, and information that changes it. Selection, calibration, and test samples each contain 4,000 independently generated outcomes with exactly balanced masks, and four selection sites each contain 1,000 observations. We use the same threshold learners and validation rule as in the clinical analysis, with 20 repetitions and no tuning. Population coverage is computed exactly from the known death probabilities, conditional on each fitted rule.

The validation rule chooses the mask learner in all 20 quantile-changing repetitions and the risk learner in all 40 others. In the changing case, final sets differ on half the population: the mask rule returns a singleton for the lower-mortality state and both labels for the other. Its mean population coverage is 0.99 and mean size 1.5, compared with coverage 1 and size 2 for risk alone. These examples establish sensitivity to this specific signal, not power against all clinically plausible alternatives, and they illustrate how atoms in the score can lead to conservative coverage.

\paragraph{Clinical threshold differences.}
In eICU, none of the ten mask learners fitted on the full selection sample splits on a mask feature, so their raw thresholds, calibrated thresholds, and final sets are identical to those of the risk learner. In MIMIC-IV, nine of 18 learners use mask splits. The mean fraction of test patients with different sets is 0.0169, with a maximum of 0.0577. The mean within-assignment maximum absolute threshold difference is 0.1178 before calibration and 0.1175 after it, and the calibrated difference averaged over patients and then assignments is 0.0149. These diagnostics distinguish a learner that does not use masks from threshold changes that affect some predictions; they do not determine whether unused masks contain clinically relevant information. Because XGBoost already receives missing numeric inputs, this analysis concerns information beyond its fitted risk, not the total value of missingness.

%% file: tables/hospital_table.tex
\begin{tabular}{lllrrrrr}
\toprule
$w_0$ & Shift & Rule & $A$ & $D$ & $B$ & $D-A$ & $B-D$ \\
\midrule
0.5 & 0.00 & Pooled & 0.0190 & 0.1190 & 0.1190 & 0.1000 & 0.0000 \\
0.5 & 0.00 & Mask & 0.0000 & 0.1000 & 0.1000 & 0.1000 & 0.0000 \\
0.5 & 0.12 & Pooled & 0.0736 & 0.1736 & 0.1736 & 0.1000 & 0.0000 \\
0.5 & 0.12 & Mask & 0.0545 & 0.1545 & 0.1545 & 0.1000 & 0.0000 \\
0.8 & 0.00 & Pooled & 0.0524 & 0.1076 & 0.1076 & 0.0552 & 0.0000 \\
0.8 & 0.00 & Mask & 0.0600 & 0.1000 & 0.1000 & 0.0400 & 0.0000 \\
0.8 & 0.12 & Pooled & 0.0306 & 0.1294 & 0.1294 & 0.0989 & 0.0000 \\
0.8 & 0.12 & Mask & 0.0600 & 0.1218 & 0.1218 & 0.0618 & 0.0000 \\
\bottomrule
\end{tabular}

%% file: tables/clinical_decomposition.tex
\begin{tabular}{llrrrrr}
\toprule
Model & Weights & $\Delta A$ & $\Delta D$ & $\Delta B$ & $\Delta(D-A)$ & $\Delta(B-D)$ \\
\midrule
LR & Patient & -0.01778 & -0.00643 & -0.00720 & +0.01135 & -0.00076 \\
LR & Equal & -0.01586 & -0.00464 & -0.00462 & +0.01121 & +0.00002 \\
XGB & Patient & +0.00192 & +0.00203 & +0.00205 & +0.00011 & +0.00002 \\
XGB & Equal & +0.00282 & +0.00259 & +0.00235 & -0.00022 & -0.00024 \\
MLP & Patient & -0.00768 & -0.00137 & -0.00066 & +0.00631 & +0.00071 \\
MLP & Equal & -0.00658 & -0.00095 & -0.00039 & +0.00563 & +0.00056 \\
\bottomrule
\end{tabular}

%% file: tables/covariance_table.tex
\begin{tabular}{lrrr}
\toprule
Mask state & Pooled calibration & Missingness calibration & Difference \\
\midrule
0 & -0.00362 & -0.00393 & -0.00032 \\
1 & +0.01382 & +0.01478 & +0.00096 \\
\bottomrule
\end{tabular}

%% file: tables/paired_intervals.tex
\begin{tabular}{llrrrrr}
\toprule
Data & Model & Gain (pp) & 95\% lower & 95\% upper & Adjusted lower & Adjusted upper \\
\midrule
eICU & LR & 1.91 & 1.73 & 3.17 & 1.55 & 3.46 \\
eICU & XGB & 0.34 & -0.08 & 1.07 & -0.22 & 1.33 \\
eICU & MLP & 1.17 & 1.09 & 1.94 & 1.00 & 2.16 \\
MIMIC-IV & LR & 3.13 & 1.76 & 4.52 & 1.33 & 5.01 \\
MIMIC-IV & XGB & 1.83 & 1.33 & 2.46 & 1.30 & 2.66 \\
MIMIC-IV & MLP & 2.87 & 1.83 & 3.97 & 1.61 & 4.40 \\
\bottomrule
\end{tabular}

%% file: tables/panel_sensitivity.tex
\begin{tabular}{llrrrr}
\toprule
Data & Model & All groups & $n\ge500$ & $n\ge1000$ & Standardized \\
\midrule
eICU & LR & +0.069 & +0.069 & +0.069 & +18.133 \\
eICU & XGB & +0.097 & +0.097 & +0.097 & +24.858 \\
eICU & MLP & +0.034 & +0.034 & +0.034 & +9.493 \\
MIMIC-IV & LR & +0.066 & +0.066 & +0.054 & +11.046 \\
MIMIC-IV & XGB & +0.037 & +0.040 & +0.041 & +7.255 \\
MIMIC-IV & MLP & +0.016 & +0.017 & +0.016 & +2.757 \\
\bottomrule
\end{tabular}

%% file: tables/all_methods.tex
\begin{longtable}{lllrrrr}
\toprule
Data & Model & Rule & Coverage & Size & Selected & Panel \\
\midrule
\endhead
eICU & LR & Pooled & 0.896 & 1.221 & 0.040 & 0.116 \\
eICU & LR & Missingness & 0.897 & 1.224 & 0.021 & 0.113 \\
eICU & LR & Risk (2) & 0.900 & 1.340 & 0.018 & 0.064 \\
eICU & LR & Risk (3) & 0.900 & 1.212 & 0.014 & 0.044 \\
eICU & LR & Label & 0.896 & 1.253 & 0.040 & 0.117 \\
eICU & LR & Risk + label & 0.901 & 1.669 & 0.008 & 0.034 \\
eICU & LR & Random & 0.900 & 1.227 & 0.031 & 0.099 \\
eICU & LR & Lactate & 0.895 & 1.208 & 0.025 & 0.092 \\
eICU & LR & Score tree & 0.899 & 1.273 & 0.013 & 0.050 \\
eICU & LR & Mask weighting & 0.897 & 1.223 & 0.040 & 0.115 \\
eICU & XGB & Pooled & 0.895 & 0.952 & 0.029 & 0.118 \\
eICU & XGB & Missingness & 0.895 & 0.952 & 0.026 & 0.116 \\
eICU & XGB & Risk (2) & 0.899 & 1.002 & 0.012 & 0.029 \\
eICU & XGB & Risk (3) & 0.900 & 1.026 & 0.009 & 0.020 \\
eICU & XGB & Label & 0.897 & 1.246 & 0.030 & 0.117 \\
eICU & XGB & Risk + label & 0.900 & 1.637 & 0.008 & 0.038 \\
eICU & XGB & Random & 0.899 & 0.966 & 0.022 & 0.090 \\
eICU & XGB & Lactate & 0.895 & 0.966 & 0.020 & 0.085 \\
eICU & XGB & Score tree & 0.900 & 1.031 & 0.013 & 0.026 \\
eICU & XGB & Mask weighting & 0.896 & 0.952 & 0.029 & 0.116 \\
eICU & MLP & Pooled & 0.898 & 1.037 & 0.029 & 0.080 \\
eICU & MLP & Missingness & 0.898 & 1.038 & 0.017 & 0.079 \\
eICU & MLP & Risk (2) & 0.899 & 1.100 & 0.017 & 0.051 \\
eICU & MLP & Risk (3) & 0.898 & 1.121 & 0.014 & 0.045 \\
eICU & MLP & Label & 0.899 & 1.500 & 0.029 & 0.067 \\
eICU & MLP & Risk + label & 0.900 & 1.750 & 0.010 & 0.028 \\
eICU & MLP & Random & 0.900 & 1.050 & 0.023 & 0.063 \\
eICU & MLP & Lactate & 0.897 & 1.049 & 0.018 & 0.060 \\
eICU & MLP & Score tree & 0.898 & 1.067 & 0.018 & 0.050 \\
eICU & MLP & Mask weighting & 0.897 & 1.036 & 0.029 & 0.080 \\
MIMIC-IV & LR & Pooled & 0.881 & 1.287 & 0.083 & 0.138 \\
MIMIC-IV & LR & Missingness & 0.885 & 1.296 & 0.052 & 0.124 \\
MIMIC-IV & LR & Risk (2) & 0.889 & 1.286 & 0.053 & 0.085 \\
MIMIC-IV & LR & Risk (3) & 0.896 & 1.213 & 0.038 & 0.058 \\
MIMIC-IV & LR & Label & 0.876 & 1.292 & 0.086 & 0.142 \\
MIMIC-IV & LR & Risk + label & 0.893 & 1.704 & 0.026 & 0.043 \\
MIMIC-IV & LR & Random & 0.879 & 1.278 & 0.076 & 0.126 \\
MIMIC-IV & LR & Lactate & 0.886 & 1.297 & 0.081 & 0.104 \\
MIMIC-IV & LR & Score tree & 0.892 & 1.297 & 0.047 & 0.075 \\
MIMIC-IV & LR & Mask weighting & 0.888 & 1.305 & 0.076 & 0.129 \\
MIMIC-IV & XGB & Pooled & 0.904 & 1.006 & 0.060 & 0.103 \\
MIMIC-IV & XGB & Missingness & 0.910 & 1.026 & 0.041 & 0.079 \\
MIMIC-IV & XGB & Risk (2) & 0.909 & 1.106 & 0.027 & 0.042 \\
MIMIC-IV & XGB & Risk (3) & 0.909 & 1.111 & 0.026 & 0.042 \\
MIMIC-IV & XGB & Label & 0.871 & 1.284 & 0.090 & 0.146 \\
MIMIC-IV & XGB & Risk + label & 0.895 & 1.611 & 0.027 & 0.048 \\
MIMIC-IV & XGB & Random & 0.904 & 1.015 & 0.047 & 0.085 \\
MIMIC-IV & XGB & Lactate & 0.900 & 1.018 & 0.064 & 0.091 \\
MIMIC-IV & XGB & Score tree & 0.907 & 1.111 & 0.028 & 0.045 \\
MIMIC-IV & XGB & Mask weighting & 0.913 & 1.022 & 0.048 & 0.083 \\
MIMIC-IV & MLP & Pooled & 0.899 & 1.092 & 0.058 & 0.099 \\
MIMIC-IV & MLP & Missingness & 0.905 & 1.111 & 0.029 & 0.081 \\
MIMIC-IV & MLP & Risk (2) & 0.902 & 1.161 & 0.039 & 0.065 \\
MIMIC-IV & MLP & Risk (3) & 0.904 & 1.168 & 0.038 & 0.065 \\
MIMIC-IV & MLP & Label & 0.889 & 1.527 & 0.059 & 0.092 \\
MIMIC-IV & MLP & Risk + label & 0.897 & 1.750 & 0.019 & 0.039 \\
MIMIC-IV & MLP & Random & 0.901 & 1.101 & 0.048 & 0.086 \\
MIMIC-IV & MLP & Lactate & 0.902 & 1.114 & 0.057 & 0.084 \\
MIMIC-IV & MLP & Score tree & 0.905 & 1.116 & 0.040 & 0.063 \\
MIMIC-IV & MLP & Mask weighting & 0.906 & 1.107 & 0.050 & 0.087 \\
\bottomrule
\end{longtable}

%% file: tables/set_disagreement.tex
\begin{tabular}{llrr}
\toprule
Data & Model & Mean changed (\%) & Assignment range (\%) \\
\midrule
eICU & LR & 4.52 & 0.98--11.09 \\
eICU & XGB & 1.03 & 0.01--3.91 \\
eICU & MLP & 1.52 & 0.66--3.07 \\
MIMIC-IV & LR & 9.38 & 5.40--17.57 \\
MIMIC-IV & XGB & 4.97 & 1.24--12.80 \\
MIMIC-IV & MLP & 7.08 & 3.18--13.29 \\
\bottomrule
\end{tabular}

%% file: tables/diagnostic_table.tex
\begin{longtable}{lllrrrrr}
\toprule
Data & Model & Rule & $U$ & $O$ & Panel gap & $C_0$ & $C_1$ \\
\midrule
\endhead
eICU & LR & Pooled & 0.019 & 0.026 & 0.116 & 0.897 & 0.888 \\
eICU & LR & Missingness & 0.014 & 0.010 & 0.113 & 0.898 & 0.885 \\
eICU & LR & Risk (2) & 0.009 & 0.012 & 0.064 & 0.896 & 0.937 \\
eICU & LR & Risk (3) & 0.006 & 0.008 & 0.044 & 0.903 & 0.868 \\
eICU & LR & Label & 0.019 & 0.026 & 0.117 & 0.895 & 0.906 \\
eICU & XGB & Pooled & 0.017 & 0.016 & 0.118 & 0.971 & 0.155 \\
eICU & XGB & Missingness & 0.017 & 0.009 & 0.116 & 0.970 & 0.161 \\
eICU & XGB & Risk (2) & 0.007 & 0.006 & 0.029 & 0.949 & 0.416 \\
eICU & XGB & Risk (3) & 0.008 & 0.003 & 0.020 & 0.939 & 0.512 \\
eICU & XGB & Label & 0.012 & 0.022 & 0.117 & 0.896 & 0.903 \\
eICU & MLP & Pooled & 0.014 & 0.019 & 0.080 & 0.944 & 0.446 \\
eICU & MLP & Missingness & 0.009 & 0.009 & 0.079 & 0.944 & 0.448 \\
eICU & MLP & Risk (2) & 0.008 & 0.012 & 0.051 & 0.934 & 0.555 \\
eICU & MLP & Risk (3) & 0.006 & 0.009 & 0.045 & 0.930 & 0.587 \\
eICU & MLP & Label & 0.013 & 0.019 & 0.067 & 0.898 & 0.907 \\
MIMIC-IV & LR & Pooled & 0.071 & 0.018 & 0.138 & 0.877 & 0.901 \\
MIMIC-IV & LR & Missingness & 0.038 & 0.016 & 0.124 & 0.881 & 0.904 \\
MIMIC-IV & LR & Risk (2) & 0.046 & 0.012 & 0.085 & 0.888 & 0.897 \\
MIMIC-IV & LR & Risk (3) & 0.032 & 0.011 & 0.058 & 0.906 & 0.825 \\
MIMIC-IV & LR & Label & 0.075 & 0.017 & 0.142 & 0.871 & 0.909 \\
MIMIC-IV & XGB & Pooled & 0.044 & 0.026 & 0.103 & 0.980 & 0.366 \\
MIMIC-IV & XGB & Missingness & 0.018 & 0.026 & 0.079 & 0.979 & 0.417 \\
MIMIC-IV & XGB & Risk (2) & 0.011 & 0.016 & 0.042 & 0.948 & 0.630 \\
MIMIC-IV & XGB & Risk (3) & 0.010 & 0.017 & 0.042 & 0.949 & 0.626 \\
MIMIC-IV & XGB & Label & 0.078 & 0.018 & 0.146 & 0.866 & 0.910 \\
MIMIC-IV & MLP & Pooled & 0.047 & 0.021 & 0.099 & 0.953 & 0.519 \\
MIMIC-IV & MLP & Missingness & 0.014 & 0.018 & 0.081 & 0.955 & 0.550 \\
MIMIC-IV & MLP & Risk (2) & 0.030 & 0.017 & 0.065 & 0.941 & 0.624 \\
MIMIC-IV & MLP & Risk (3) & 0.028 & 0.018 & 0.065 & 0.942 & 0.632 \\
MIMIC-IV & MLP & Label & 0.051 & 0.014 & 0.092 & 0.886 & 0.909 \\
\bottomrule
\end{longtable}

%% file: tables/composition_table.tex
\begin{longtable}{lllrrrr}
\toprule
Data & Model & Rule & $\varnothing$ & $\{0\}$ & $\{1\}$ & $\{0,1\}$ \\
\midrule
\endhead
eICU & LR & Pooled & 0.000 & 0.625 & 0.153 & 0.221 \\
eICU & LR & Missingness & 0.000 & 0.624 & 0.152 & 0.224 \\
eICU & LR & Risk (2) & 0.043 & 0.473 & 0.100 & 0.383 \\
eICU & LR & Risk (3) & 0.054 & 0.603 & 0.076 & 0.267 \\
eICU & LR & Label & 0.000 & 0.592 & 0.155 & 0.253 \\
eICU & XGB & Pooled & 0.048 & 0.935 & 0.017 & 0.000 \\
eICU & XGB & Missingness & 0.048 & 0.935 & 0.017 & 0.000 \\
eICU & XGB & Risk (2) & 0.045 & 0.891 & 0.017 & 0.047 \\
eICU & XGB & Risk (3) & 0.056 & 0.852 & 0.011 & 0.082 \\
eICU & XGB & Label & 0.000 & 0.598 & 0.156 & 0.246 \\
eICU & MLP & Pooled & 0.001 & 0.878 & 0.084 & 0.038 \\
eICU & MLP & Missingness & 0.003 & 0.874 & 0.082 & 0.041 \\
eICU & MLP & Risk (2) & 0.034 & 0.778 & 0.053 & 0.134 \\
eICU & MLP & Risk (3) & 0.042 & 0.747 & 0.046 & 0.164 \\
eICU & MLP & Label & 0.000 & 0.363 & 0.136 & 0.500 \\
MIMIC-IV & LR & Pooled & 0.000 & 0.529 & 0.184 & 0.287 \\
MIMIC-IV & LR & Missingness & 0.000 & 0.526 & 0.178 & 0.296 \\
MIMIC-IV & LR & Risk (2) & 0.047 & 0.505 & 0.115 & 0.333 \\
MIMIC-IV & LR & Risk (3) & 0.042 & 0.609 & 0.094 & 0.255 \\
MIMIC-IV & LR & Label & 0.000 & 0.518 & 0.190 & 0.292 \\
MIMIC-IV & XGB & Pooled & 0.014 & 0.924 & 0.042 & 0.020 \\
MIMIC-IV & XGB & Missingness & 0.019 & 0.902 & 0.034 & 0.045 \\
MIMIC-IV & XGB & Risk (2) & 0.046 & 0.783 & 0.018 & 0.153 \\
MIMIC-IV & XGB & Risk (3) & 0.044 & 0.782 & 0.019 & 0.155 \\
MIMIC-IV & XGB & Label & 0.000 & 0.519 & 0.197 & 0.284 \\
MIMIC-IV & MLP & Pooled & 0.000 & 0.825 & 0.083 & 0.092 \\
MIMIC-IV & MLP & Missingness & 0.004 & 0.808 & 0.074 & 0.114 \\
MIMIC-IV & MLP & Risk (2) & 0.030 & 0.724 & 0.056 & 0.191 \\
MIMIC-IV & MLP & Risk (3) & 0.029 & 0.718 & 0.056 & 0.198 \\
MIMIC-IV & MLP & Label & 0.000 & 0.311 & 0.161 & 0.527 \\
\bottomrule
\end{longtable}

%% file: tables/conditional_table.tex
\begin{tabular}{llrrrrr}
\toprule
Data & Rule & Selected gap & Panel gap & $U$ & $C$ & Size \\
\midrule
eICU & Missingness & 0.031 & 0.125 & 0.125 & 0.893 & 0.952 \\
eICU & Risk & 0.016 & 0.039 & 0.032 & 0.899 & 1.026 \\
eICU & Conditional & 0.036 & 0.055 & 0.022 & 0.900 & 1.016 \\
eICU & Conditional + mask & 0.025 & 0.051 & 0.031 & 0.899 & 1.015 \\
MIMIC-IV & Missingness & 0.045 & 0.087 & 0.068 & 0.906 & 1.024 \\
MIMIC-IV & Risk & 0.029 & 0.052 & 0.025 & 0.906 & 1.110 \\
MIMIC-IV & Conditional & 0.038 & 0.056 & 0.029 & 0.917 & 1.081 \\
MIMIC-IV & Conditional + mask & 0.040 & 0.058 & 0.028 & 0.918 & 1.083 \\
\bottomrule
\end{tabular}

%% file: tables/prevalence_table.tex
\begin{tabular}{lllrrrr}
\toprule
Data & Sample & Population & Cells & $n$ range & $\mu$ range & $N_{\rm empty}$ \\
\midrule
eICU & calibration & All & 10 & 10542--12247 & 0.071--0.101 & 0 \\
eICU & calibration & Mask & 20 & 783--11139 & 0.064--0.126 & 0 \\
eICU & calibration & Site + mask & 172 & 0--4333 & 0.016--0.266 & 7 \\
eICU & test & All & 10 & 20097--22316 & 0.081--0.109 & 0 \\
eICU & test & Mask & 20 & 2045--19316 & 0.065--0.123 & 0 \\
eICU & test & Site + mask & 342 & 0--3607 & 0.000--0.438 & 9 \\
MIMIC-IV & calibration & All & 18 & 7871--12851 & 0.106--0.174 & 0 \\
MIMIC-IV & calibration & Mask & 36 & 1211--10135 & 0.069--0.330 & 0 \\
MIMIC-IV & calibration & Site + mask & 36 & 1211--10135 & 0.069--0.330 & 0 \\
MIMIC-IV & test & All & 18 & 7227--12851 & 0.036--0.174 & 0 \\
MIMIC-IV & test & Mask & 36 & 1357--10215 & 0.020--0.330 & 0 \\
MIMIC-IV & test & Site + mask & 36 & 1357--10215 & 0.020--0.330 & 0 \\
\bottomrule
\end{tabular}

%% file: tables/population_outcomes.tex
\begin{tabular}{lllrrrr}
\toprule
Data & Model & Population & $C$ & $C_0$ & $C_1$ & Bound \\
\midrule
eICU & LR & Mask & 0.896 & 0.898 & 0.874 & 0.033 \\
eICU & LR & Site + mask & 0.899 & 0.899 & 0.857 & 0.153 \\
eICU & MLP & Mask & 0.900 & 0.945 & 0.448 & 0.018 \\
eICU & MLP & Site + mask & 0.897 & 0.943 & 0.445 & 0.082 \\
eICU & XGB & Mask & 0.895 & 0.970 & 0.157 & 0.015 \\
eICU & XGB & Site + mask & 0.892 & 0.969 & 0.163 & 0.052 \\
MIMIC-IV & LR & Mask & 0.885 & 0.876 & 0.905 & 0.270 \\
MIMIC-IV & LR & Site + mask & 0.885 & 0.876 & 0.905 & 0.270 \\
MIMIC-IV & MLP & Mask & 0.903 & 0.954 & 0.556 & 0.298 \\
MIMIC-IV & MLP & Site + mask & 0.903 & 0.954 & 0.556 & 0.298 \\
MIMIC-IV & XGB & Mask & 0.908 & 0.980 & 0.426 & 0.336 \\
MIMIC-IV & XGB & Site + mask & 0.908 & 0.980 & 0.426 & 0.336 \\
\bottomrule
\end{tabular}

%% file: tables/alpha_table.tex
\begin{tabular}{lllrrrrr}
\toprule
Data & Model & $\alpha$ & $C$ & $C_0$ & $C_1$ & Size & $U$ \\
\midrule
eICU & LR & 0.05 & 0.948 & 0.948 & 0.950 & 1.433 & 0.071 \\
eICU & LR & 0.10 & 0.897 & 0.898 & 0.885 & 1.224 & 0.112 \\
eICU & LR & 0.20 & 0.797 & 0.800 & 0.770 & 0.985 & 0.163 \\
eICU & MLP & 0.05 & 0.948 & 0.981 & 0.631 & 1.209 & 0.044 \\
eICU & MLP & 0.10 & 0.898 & 0.944 & 0.448 & 1.038 & 0.079 \\
eICU & MLP & 0.20 & 0.797 & 0.851 & 0.277 & 0.867 & 0.124 \\
eICU & XGB & 0.05 & 0.949 & 0.998 & 0.470 & 1.066 & 0.031 \\
eICU & XGB & 0.10 & 0.895 & 0.970 & 0.161 & 0.952 & 0.116 \\
eICU & XGB & 0.20 & 0.795 & 0.872 & 0.044 & 0.823 & 0.219 \\
MIMIC-IV & LR & 0.05 & 0.942 & 0.940 & 0.954 & 1.507 & 0.079 \\
MIMIC-IV & LR & 0.10 & 0.885 & 0.881 & 0.904 & 1.296 & 0.114 \\
MIMIC-IV & LR & 0.20 & 0.777 & 0.772 & 0.807 & 1.032 & 0.140 \\
MIMIC-IV & MLP & 0.05 & 0.955 & 0.985 & 0.740 & 1.328 & 0.042 \\
MIMIC-IV & MLP & 0.10 & 0.905 & 0.955 & 0.550 & 1.111 & 0.072 \\
MIMIC-IV & MLP & 0.20 & 0.800 & 0.863 & 0.360 & 0.904 & 0.113 \\
MIMIC-IV & XGB & 0.05 & 0.959 & 0.998 & 0.683 & 1.190 & 0.016 \\
MIMIC-IV & XGB & 0.10 & 0.910 & 0.979 & 0.417 & 1.026 & 0.061 \\
MIMIC-IV & XGB & 0.20 & 0.789 & 0.875 & 0.185 & 0.845 & 0.179 \\
\bottomrule
\end{tabular}

%% file: tables/scale_table.tex
\begin{tabular}{lllrrrrr}
\toprule
Data & Model & Odds factor & $C$ & $C_0$ & $C_1$ & Size & $U$ \\
\midrule
eICU & LR & 0.25 & 0.896 & 0.939 & 0.473 & 0.985 & 0.104 \\
eICU & LR & 1.00 & 0.897 & 0.898 & 0.885 & 1.224 & 0.112 \\
eICU & LR & 4.00 & 0.896 & 0.885 & 0.998 & 1.751 & 0.124 \\
eICU & MLP & 0.25 & 0.898 & 0.950 & 0.391 & 1.018 & 0.082 \\
eICU & MLP & 1.00 & 0.898 & 0.944 & 0.448 & 1.038 & 0.079 \\
eICU & MLP & 4.00 & 0.898 & 0.938 & 0.510 & 1.064 & 0.077 \\
eICU & XGB & 0.25 & 0.894 & 0.984 & 0.024 & 0.962 & 0.121 \\
eICU & XGB & 1.00 & 0.895 & 0.970 & 0.161 & 0.952 & 0.116 \\
eICU & XGB & 4.00 & 0.896 & 0.927 & 0.593 & 1.004 & 0.104 \\
MIMIC-IV & LR & 0.25 & 0.901 & 0.934 & 0.653 & 1.076 & 0.084 \\
MIMIC-IV & LR & 1.00 & 0.885 & 0.881 & 0.904 & 1.296 & 0.114 \\
MIMIC-IV & LR & 4.00 & 0.876 & 0.861 & 0.992 & 1.687 & 0.126 \\
MIMIC-IV & MLP & 0.25 & 0.906 & 0.963 & 0.507 & 1.092 & 0.069 \\
MIMIC-IV & MLP & 1.00 & 0.905 & 0.955 & 0.550 & 1.111 & 0.072 \\
MIMIC-IV & MLP & 4.00 & 0.903 & 0.946 & 0.603 & 1.136 & 0.072 \\
MIMIC-IV & XGB & 0.25 & 0.914 & 0.995 & 0.335 & 1.048 & 0.055 \\
MIMIC-IV & XGB & 1.00 & 0.910 & 0.979 & 0.417 & 1.026 & 0.061 \\
MIMIC-IV & XGB & 4.00 & 0.893 & 0.919 & 0.715 & 1.096 & 0.092 \\
\bottomrule
\end{tabular}

%% file: tables/six_cells.tex
\begin{tabular}{lrrrrrr}
\toprule
Data & Risk & Label & Risk only & Label only & Joint & Joint range \\
\midrule
eICU & 1 & 0 & 0.905 & 1.000 & 0.900 & 0.887--0.914 \\
eICU & 1 & 1 & 0.000 & 0.000 & 0.916 & 0.828--1.000 \\
eICU & 2 & 0 & 0.928 & 1.000 & 0.901 & 0.894--0.913 \\
eICU & 2 & 1 & 0.000 & 0.308 & 0.921 & 0.873--0.954 \\
eICU & 3 & 0 & 0.997 & 0.638 & 0.895 & 0.873--0.921 \\
eICU & 3 & 1 & 0.589 & 1.000 & 0.906 & 0.883--0.924 \\
MIMIC-IV & 1 & 0 & 0.901 & 1.000 & 0.886 & 0.831--0.936 \\
MIMIC-IV & 1 & 1 & 0.000 & 0.015 & 0.892 & 0.750--1.000 \\
MIMIC-IV & 2 & 0 & 0.974 & 1.000 & 0.893 & 0.857--0.940 \\
MIMIC-IV & 2 & 1 & 0.043 & 0.662 & 0.901 & 0.818--0.958 \\
MIMIC-IV & 3 & 0 & 0.987 & 0.442 & 0.901 & 0.852--0.970 \\
MIMIC-IV & 3 & 1 & 0.777 & 1.000 & 0.905 & 0.792--0.954 \\
\bottomrule
\end{tabular}

%% file: tables/joint_table.tex
\begin{tabular}{lllrrrr}
\toprule
Data & Model & Rule & $C$ & $C_1$ & $U$ & Size \\
\midrule
eICU & LR & Label & 0.896 & 0.906 & 0.116 & 1.253 \\
eICU & LR & Risk + label & 0.901 & 0.904 & 0.034 & 1.669 \\
eICU & LR & Risk (3) & 0.900 & 0.868 & 0.043 & 1.212 \\
eICU & MLP & Label & 0.899 & 0.907 & 0.065 & 1.500 \\
eICU & MLP & Risk + label & 0.900 & 0.906 & 0.026 & 1.750 \\
eICU & MLP & Risk (3) & 0.898 & 0.587 & 0.045 & 1.121 \\
eICU & XGB & Label & 0.897 & 0.903 & 0.117 & 1.246 \\
eICU & XGB & Risk + label & 0.900 & 0.908 & 0.038 & 1.637 \\
eICU & XGB & Risk (3) & 0.900 & 0.512 & 0.019 & 1.026 \\
MIMIC-IV & LR & Label & 0.876 & 0.909 & 0.133 & 1.292 \\
MIMIC-IV & LR & Risk + label & 0.893 & 0.910 & 0.038 & 1.704 \\
MIMIC-IV & LR & Risk (3) & 0.896 & 0.825 & 0.052 & 1.213 \\
MIMIC-IV & MLP & Label & 0.889 & 0.909 & 0.085 & 1.527 \\
MIMIC-IV & MLP & Risk + label & 0.897 & 0.916 & 0.032 & 1.750 \\
MIMIC-IV & MLP & Risk (3) & 0.904 & 0.632 & 0.058 & 1.168 \\
MIMIC-IV & XGB & Label & 0.871 & 0.910 & 0.138 & 1.284 \\
MIMIC-IV & XGB & Risk + label & 0.895 & 0.906 & 0.043 & 1.611 \\
MIMIC-IV & XGB & Risk (3) & 0.909 & 0.626 & 0.019 & 1.111 \\
\bottomrule
\end{tabular}

%% file: tables/cell_counts.tex
\begin{tabular}{lrrrrrr}
\toprule
Data & Risk & Label & $n_{\rm cal}$ & Rank & $n_{\rm test}$ & Max rank \\
\midrule
eICU & 1 & 0 & 3042--4418 & 2739--3978 & 6331--7797 & 0 \\
eICU & 1 & 1 & 18--35 & 18--33 & 23--60 & 1 \\
eICU & 2 & 0 & 3100--3835 & 2791--3453 & 5854--7621 & 0 \\
eICU & 2 & 1 & 81--158 & 74--144 & 171--302 & 0 \\
eICU & 3 & 0 & 2245--3443 & 2022--3100 & 4966--6595 & 0 \\
eICU & 3 & 1 & 622--1088 & 561--981 & 1372--2051 & 0 \\
MIMIC-IV & 1 & 0 & 3024--4944 & 2723--4451 & 1992--5012 & 0 \\
MIMIC-IV & 1 & 1 & 18--150 & 18--136 & 8--155 & 1 \\
MIMIC-IV & 2 & 0 & 1904--4056 & 1715--3652 & 1972--4442 & 0 \\
MIMIC-IV & 2 & 1 & 108--389 & 99--351 & 37--409 & 0 \\
MIMIC-IV & 3 & 0 & 723--3473 & 652--3127 & 736--4725 & 0 \\
MIMIC-IV & 3 & 1 & 553--1991 & 499--1793 & 331--1987 & 0 \\
\bottomrule
\end{tabular}

%% file: tables/target_composition.tex
\begin{tabular}{lllrrrr}
\toprule
Data & Model & Rule & $\varnothing$ & $\{0\}$ & $\{1\}$ & $\{0,1\}$ \\
\midrule
eICU & LR & Label & 0.000 & 0.592 & 0.155 & 0.253 \\
eICU & LR & Risk + label & 0.000 & 0.202 & 0.129 & 0.669 \\
eICU & LR & Risk (3) & 0.054 & 0.603 & 0.076 & 0.267 \\
eICU & MLP & Label & 0.000 & 0.363 & 0.136 & 0.500 \\
eICU & MLP & Risk + label & 0.000 & 0.128 & 0.122 & 0.750 \\
eICU & MLP & Risk (3) & 0.042 & 0.747 & 0.046 & 0.164 \\
eICU & XGB & Label & 0.000 & 0.598 & 0.156 & 0.246 \\
eICU & XGB & Risk + label & 0.000 & 0.228 & 0.135 & 0.637 \\
eICU & XGB & Risk (3) & 0.056 & 0.852 & 0.011 & 0.082 \\
MIMIC-IV & LR & Label & 0.000 & 0.518 & 0.190 & 0.292 \\
MIMIC-IV & LR & Risk + label & 0.000 & 0.154 & 0.142 & 0.704 \\
MIMIC-IV & LR & Risk (3) & 0.042 & 0.609 & 0.094 & 0.255 \\
MIMIC-IV & MLP & Label & 0.000 & 0.311 & 0.161 & 0.527 \\
MIMIC-IV & MLP & Risk + label & 0.000 & 0.117 & 0.133 & 0.750 \\
MIMIC-IV & MLP & Risk (3) & 0.029 & 0.718 & 0.056 & 0.198 \\
MIMIC-IV & XGB & Label & 0.000 & 0.519 & 0.197 & 0.284 \\
MIMIC-IV & XGB & Risk + label & 0.000 & 0.245 & 0.143 & 0.611 \\
MIMIC-IV & XGB & Risk (3) & 0.044 & 0.782 & 0.019 & 0.155 \\
\bottomrule
\end{tabular}

%% file: tables/residual_calibration.tex
\begin{tabular}{llrrrrr}
\toprule
Data & Rule & Coverage & Size & $U$ & Site $U$ & $C_1$ \\
\midrule
eICU & Missingness & 0.895 & 0.952 & 0.116 & 0.194 & 0.161 \\
eICU & Risk bins & 0.900 & 1.026 & 0.019 & 0.073 & 0.512 \\
eICU & Learned risk & 0.898 & 1.010 & 0.040 & 0.098 & 0.457 \\
eICU & + Mask & 0.898 & 1.010 & 0.040 & 0.098 & 0.457 \\
eICU & Gated & 0.898 & 1.010 & 0.040 & 0.098 & 0.457 \\
eICU & Label & 0.897 & 1.246 & 0.117 & 0.169 & 0.903 \\
\midrule
MIMIC-IV & Missingness & 0.910 & 1.026 & 0.061 & 0.061 & 0.417 \\
MIMIC-IV & Risk bins & 0.909 & 1.111 & 0.019 & 0.019 & 0.626 \\
MIMIC-IV & Learned risk & 0.910 & 1.040 & 0.051 & 0.051 & 0.447 \\
MIMIC-IV & + Mask & 0.909 & 1.043 & 0.050 & 0.050 & 0.462 \\
MIMIC-IV & Gated & 0.909 & 1.042 & 0.051 & 0.051 & 0.459 \\
MIMIC-IV & Label & 0.871 & 1.284 & 0.138 & 0.138 & 0.910 \\
\bottomrule
\end{tabular}